%% file: FastEmbedding.tex
\documentclass[11pt]{article}

\usepackage{paralist}
\usepackage{enumitem}
\usepackage{graphicx}
\usepackage{amsmath,amsthm}
\usepackage{amsfonts}
\usepackage{amssymb}
\usepackage{amsthm}
\usepackage{multicol}
\usepackage{latexsym}
\usepackage[top=1in, bottom=1in, left=1in, right=1in]{geometry}
\usepackage[noend]{algorithmic}
\usepackage[ruled]{algorithm2e}
\usepackage{float}
\usepackage{tikz}
\usetikzlibrary{arrows,arrows.meta}
\usetikzlibrary{decorations.markings}
\usetikzlibrary{decorations.pathmorphing}
\usetikzlibrary{positioning,shapes,backgrounds}
\usetikzlibrary{calc,intersections,patterns,fit,automata}
\usepackage{tikzsymbols}
\usepackage{pgfplots}

\input{macros}
\def\Prob{\mathbb{P}}
\def\Exp{\mathbb{E}}
\def\rvX{\mathbf{X}}
\def\rvC{\mathbf{C}}
\newcommand{\mat}[1]{{\rm#1}}
\newcommand{\remove}[1]{}
\usepackage{tikz,bm}

\newcommand{\transp}{^{\textsc{t}}}

\title{Fast Fixed Dimension \math{\ell_2}-Subspace Embeddings
  of Arbitrary Accuracy, With Application to \math{\ell_1}
  and~\math{\ell_2} Tasks}
\author{Malik Magdon-Ismail\\RPI CS Department,\\ Troy, NY 12180. \and Alex Gittens\\RPI CS Department,\\ Troy, NY 12180.}

\begin{document}

\maketitle

\begin{abstract}
  \noindent
\input{abstract.tex}
\end{abstract}

\input{intro.tex}

\input{Related}

\input{Experiments}

\input{L2.tex}
\input{L1.tex}

\bibliographystyle{plain}
\bibliography{Local,mypapers}

\end{document}

%% file: macros
\def\math#1{$#1$}

\def\mand#1{$$#1$$}
\def\mandc#1{\mand{\abovedisplayskip=3pt plus 1pt minus 1pt%
\abovedisplayshortskip=0pt plus 1pt minus 1pt%
\belowdisplayskip=3pt plus 1pt minus 1pt%
\belowdisplayshortskip=0pt plus 1pt minus 1pt%
#1}}
\def\v#1{{\mathbf #1}}
\def\frac#1#2{{#1\over #2}}

\def\mld#1{\begin{equation}
#1
\end{equation}}
\def\mldc#1{\mld{\abovedisplayskip=3pt plus 1pt minus 1pt%
\abovedisplayshortskip=0pt plus 1pt minus 1pt%
\belowdisplayskip=3pt plus 1pt minus 1pt%
\belowdisplayshortskip=0pt plus 1pt minus 1pt%
#1}}
\def\eqar#1{\begin{eqnarray}
#1
\end{eqnarray}}
\def\eqan#1{\begin{eqnarray*}
#1
\end{eqnarray*}}

\DeclareSymbolFont{AMSb}{U}{msb}{m}{n}
\DeclareMathSymbol{\N}{\mathbin}{AMSb}{"4E}
\DeclareMathSymbol{\Z}{\mathbin}{AMSb}{"5A}
\DeclareMathSymbol{\R}{\mathbin}{AMSb}{"52}
\DeclareMathSymbol{\Q}{\mathbin}{AMSb}{"51}
\DeclareMathSymbol{\I}{\mathbin}{AMSb}{"49}
\DeclareMathSymbol{\C}{\mathbin}{AMSb}{"43}

\def\choose#1#2{\left({{#1}\atop{#2}}\right)}
\def\cl#1{{\cal #1}}

\def\aa{{\mathbf a}}
\def\bb{{\mathbf b}}
\def\cc{{\mathbf c}}

\def\ee{{\mathbf e}}

\def\uu{{\mathbf u}}
\def\vv{{\mathbf v}}
\def\ww{{\mathbf w}}
\def\xx{{\mathbf x}}
\def\yy{{\mathbf y}}
\def\zz{{\mathbf z}}

\def\Z{{\mathbf Z}}

\def\exp{\mathop{\rm exp}}
\def\rank{\mathop{\rm rank}}

\def\norm#1{{\|#1\|}}
\def\trace{\mathop{\rm trace}}
\def\floor#1{{\left\lfloor\,#1\,\right\rfloor}}

\def\r#1{{(\ref{#1})}}

\newtheorem{theorem}{\bf Theorem}[section]
\newtheorem{lemma}[theorem]{Lemma}
\newtheorem{definition}[theorem]{Definition}

\def\Atilde{{\tilde\matA}}

\def\matA{{\mat{A}}}
\def\matB{{\mat{B}}}
\def\matC{{\mat{C}}}
\def\matD{{\mat{D}}}
\def\matG{{\mat{G}}}
\def\matH{{\mat{H}}}
\def\matI{{\mat{I}}}
\def\matP{{\mat{P}}}
\def\matQ{{\mat{Q}}}
\def\matR{{\mat{R}}}
\def\matS{{\mat{S}}}
\def\matU{{\mat{U}}}
\def\matV{{\mat{V}}}
\def\matW{{\mat{W}}}
\def\matX{{\mat{X}}}
\def\matZ{{\mat{Z}}}

\newcounter{exercisenum}

\def\dotfil{\leaders\hbox to 1.5mm{.}\hfill}
\newcounter{rmnum}
\def\RN#1{\setcounter{rmnum}{#1}\uppercase\expandafter{\romannumeral\value{rmnum}}}
\def\rn#1{\setcounter{rmnum}{#1}\expandafter{\romannumeral\value{rmnum}}}

\newcommand{\qedsymb}{\hfill{\rule{2mm}{2mm}}}

%% file: abstract.tex
  We give a fast oblivious \math{\ell_2}-embedding
  of \math{\matA\in\R^{n\times d}} to \math{\Atilde\in\R^{r\times d}} satisfying
  \mandc{
      (1-\varepsilon)\norm{\matA\xx}_2^2\le\norm{\Atilde\xx}_2^2\le 
    (1+\varepsilon)\norm{\matA\xx}_2^2.
  }
  Our embedding dimension \math{r} equals \math{d}, a constant independent
  of the distortion~\math{\varepsilon}. We use as a
  black-box any \math{\ell_2}-embedding \math{\Pi\transp\matA} and inherit
  its runtime and accuracy, effectively decoupling the dimension \math{r}
  from runtime and accuracy, allowing 
  downstream machine learning applications to benefit
  from both a low dimension and high accuracy (in prior
  embeddings higher accuracy means higher dimension). We give
  applications of our \math{\ell_2} embedding to regression, PCA and
  statistical leverage scores. We also give applications to \math{\ell_1}: 
  (\rn{1})
  An oblivious \math{\ell_1} embedding with
  dimension \math{d+O(d\ln^{1+\eta} d)} and distortion
  \math{O((d\ln d)/\ln\ln d)}, with application to constructing well
  -conditioned bases;
  (\rn{2}) Fast approximation of \math{\ell_1} Lewis weights
  using our \math{\ell_2} embedding to quickly approximate
  \math{\ell_2}-leverage
  scores.

%% file: intro.tex
\section{Introduction}
Sketching via
a random projection is a staple tool in modern big-data machine
learning \cite{W2014} because many
algorithms are based on metric properties of the data
(SVM, PCA, Regression, Nearest-Neighbor-Rule, etc.). So, given a big-data
matrix \math{\matA\in\R^{n\times d}} and a sketch \math{\Atilde\in\R^{r\times d}}
which approximates the metric properties of \math{\matA}, one can
approximately recover
the learning from \math{\matA} by performing the actual learning
much more efficiently on \math{\Atilde}. The general property required of
\math{\Atilde} is that it be an isometry for \math{\matA},
\mld{
  (1-\varepsilon)\norm{\matA\xx}_2^2\le\norm{\Atilde\xx}_2^2\le 
  (1+\varepsilon)\norm{\matA\xx}_2^2,\label{eq:isometry}}
where \math{\varepsilon} is the distortion and \math{\norm{\cdot}_2} is
the Euclidean norm.
The typical focus has been on linear sketches,
\math{\Atilde=\Pi\transp\matA}, where \math{\Pi\in\R^{n\times r}} is
a random draw from a distribution on \math{n\times r} matrices.
If \math{\Pi\transp} projects the columns of \math{\matA} onto a random
\math{r}-dimensional subspace, then \math{\Atilde} is an isometry
with distortion \math{\varepsilon\in \Theta(\sqrt{d/r})}.
Projecting onto a random subspace is slow, taking time
\math{\Omega(nr^2)},
and several fast approximations have evolved:
\begin{itemize}\itemsep0pt
\item
\emph{Gaussian Random Projection: \math{\Atilde=\Pi_\textsc{g}\transp\matA}.}
The entries in \math{\Pi_\textsc{g}} are chosen independently 
as Gaussians with zero mean and variance \math{1/\sqrt{r}}.
\math{\Atilde} is an isometry with \math{\varepsilon
  \in \Theta(\sqrt{d/r})}, almost as good as a random subspace.
(The Gaussians can be replaced by random signs \cite{A2001}.)
The intuition is that a matrix of independent Gaussians is an approximately
orthonormal random basis.
Computing
\math{\Atilde} takes \math{O(ndr)} time, which is still slow when
\math{n} and \math{r} are large.
\item
  \emph{Fast Subsampled Random Hadamard Transform:}
  \math{\Atilde=\Pi_\textsc{h}\transp\matA=\text{sample}_r(\matH\matD\matA/\sqrt{r})}, which
  is a uniformly random sample of
  \math{r} rescaled-rows of \math{\matH\matD\matA}, where \math{\matH}
  is an orthogonal \math{n\times n}
  Hadamard matrix consisting of \math{\pm1}s
  and \math{\matD} is a diagonal matrix of random signs.
  The intuition is that \math{\matD\matA} ``looks'' random 
  with respect to the fixed orthonormal basis \math{\matH}.
  The
  time to compute \math{\Atilde} is \math{O(nd\log_2 r)} \cite{AL2013}
  and the distortion \math{\varepsilon\in\Theta(\sqrt{(d\ln d)/r})}.
  Unfortunately, the
  \math{\sqrt{\ln d}}-factor increase in distortion is unavoidable, see
  \cite{T2011}.
For \math{d\ll n \ll e^d},
the Hadamard transform is asymptotically faster than
Gaussian random projection, but requires a slightly larger \math{r} to obtain
a comparable embedding.
\item \emph{Sparse Projections:}
  In a further effort to improve runtime,
  it is convenient for the sketching matrix \math{\Pi\transp}
  to be sparse (having few non-zeros in each column).
  The \textsf{CountSketch} projection
  \math{\Pi_\textsc{c}\transp}~\cite{WW2018}
  has one non-zero in each
  column, a random sign at a random position.
  Computing the \textsf{CountSketch} projection has runtime
  \math{\tilde O(\textsc{nnz}(\matA))+\text{poly}(d,\varepsilon)} and gives
  distortion
    \math{\varepsilon\in\Theta(\sqrt{d^2/r})}. Similarly,
    the \textsf{OSNAP} projection \math{\Pi_\textsc{o}}~\cite{NN2013}
    has \math{s=O(\log_Bd)} non-zeros in each column,
    a random sign scaled by \math{1/\sqrt{s}} at \math{s}
    random positions. The runtime for the \textsf{OSNAP} projection is also
    \math{\tilde O(\textsc{nnz}(\matA))+\text{poly}(d,\varepsilon)}
    and the distortion is \math{\varepsilon\in\Theta(\sqrt{(B\cdot d\ln d)/r})}.
\end{itemize}  
Choosing \math{r}, the dimension of the embedding, is a double-edged sword.
Larger \math{r} gives
lower distortion, improving the
accuracy of the downstream machine learning which uses \math{\Atilde} instead
of \math{\matA} (\math{\Atilde} is a coreset of \math{r} data points
obtained from \math{\matA}).
But, larger \math{r} means it takes longer to compute the embedding. 
More importantly, larger \math{r} means
slower downstream runtime, which,
depending on the application, can be super-linear in \math{r}.
We take \math{r} as fixed by feasibility considerations for
the downstream application.
The task now is to
obtain the best possible accuracy (smallest distortion) given \math{r},
as quickly as possible.

A particularly bad data matrix for the randomized Hadamard is
\mld{
  \matA\transp
  =
  \left[
    \begin{matrix}
      \ee_1\ee_1\transp\\
      \ee_2\ee_2\transp\\
      \vdots\\
      \ee_d\ee_d\transp
    \end{matrix}
  \right]
  \in\R^{d\times d^2},
\label{eq:badA}}
which, via a coupon collecting argument, realizes the
\math{\sqrt{\ln d}}-factor increase in distortion~\cite{T2011}.
After adding a little random noise to each entry of
\math{\matA}, we show how the 
runtime and distortion for the fast Hadamard transform
depend on the embedding dimension \math{r} in
Figure~\ref{fig:distortion} (Figure~\ref{fig:distortion}(a) is runtime
and Figure~\ref{fig:distortion}(b) is distortion. The red curves
compare speed up and distortion of the fast Hadamard 
to the Gaussian random projection):
there is a considerable speedup even for small \math{r},
however the distortion is about 50\% larger for the
fast Hadamard as compared to the Gaussian. There are fast ways to recover
a more Gaussian-like embedding using fast Hadamard transforms,
\emph{without increasing the embedding dimension}. One approach is to iterate
the Hadamard transform a number of times~\cite{AR2014,KW2011},
for example \math{\Atilde=\text{sample}_r(\matH\matD_2\matH\matD_1\matA/\sqrt{r})}
where \math{\matD_1, \matD_2} are independent diagonal matrices of random
signs. With each application of the Hadamard one gets ``closer'' to projection
onto a random basis. The results are the green curves in
Figure~\ref{fig:distortion}. The speedup is halved with two
Hadamard steps, but the distortion is
considerably improved, even slightly better than the Gaussian embedding.
This is because the Gaussian embedding is not quite projection onto a random
orthonormal basis (there are both scaling and orthogonality discrepancies).

Our contribution is to give a simple, \emph{non-linear}, fast oblivious
embedding of
\math{\matA} to a \emph{fixed}-dimension \math{d} with distortion
\math{\varepsilon}, and its applications to some standard machine learning
applications: \math{\ell_2}-regression, low rank matrix reconstruction (PCA),
fast estimation of \math{\ell_2}-leverage scores, oblivious
\math{\ell_1}-embedding
and corresponding \math{\ell_1}-applications.
Our \math{\ell_2}-embedding is based the simple observation when you hit
\math{\matA} with a Gaussian random projection
you obtain a random matrix
\math{\Pi_{n\times r}\transp\matA}, which has the \emph{same} distribution as
\math{\Pi_{d\times r}\transp(\matA\transp\matA)^{1/2}},
\mld{
  \Pi_{n\times r}\transp\matA\sim \Pi_{d\times r}\transp(\matA\transp\matA)^{1/2}.
  }%
\begin{figure}[t]
{\tabcolsep0pt
  \begin{tabular}{p{0.5\textwidth}p{0.5\textwidth}}
\scalebox{0.5}{\input{RunTime.tex}}
&
\scalebox{0.5}{\input{Distortion.tex}}
\\
\multicolumn{1}{c}{(a) Speedup over Gaussian embedding}&
\multicolumn{1}{c}{(b) \math{\varepsilon=\text{distortion}},
  \math{\varepsilon_\textsc{g}=\text{Gaussian distortion}}.}
  \end{tabular}
  }
\caption{Comparison of various fast embeddings with the Gaussian embedding.
  We use the matrix \math{\matA} as given in \r{eq:badA} with \math{d=2^9}.
  In (b), the distortion of \math{\Atilde} is computed as
  \math{\varepsilon
    =
    \norm{\matI-(\matA\transp\matA)^{-1/2}
      \Atilde\transp\Atilde
      (\matA\transp\matA)^{-1/2}}_2}.
  \label{fig:distortion}}
\end{figure}
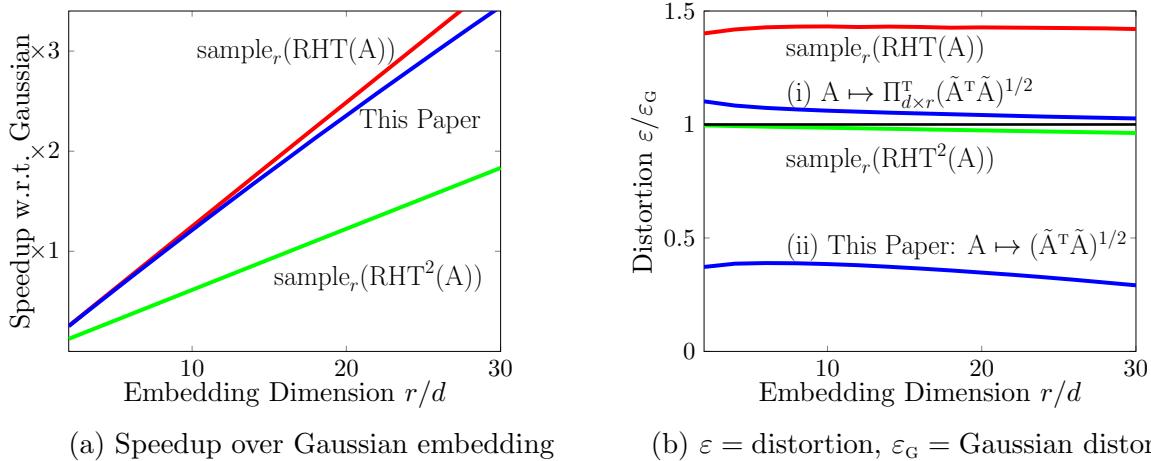%
Both the LHS and the RHS give an embedding with
\emph{identical} statistical properties. To see this, observe
that for both the LHS and RHS, each row is an independent identically
distributed Gaussian vector, so it suffices to compare covariance
matrices. A simple calculation shows that the row-covariance matrix
for both constructions is \math{\matA\transp\matA}.
However the LHS takes time
\math{O(nrd)} to compute using \math{nr} draws from the standard normal
distribution, while the RHS takes time
\math{O(nd^2+d^3+rd^2)} using \math{dr} draws from the standard normal
distribution.\footnote{Here, for simplicity of exposition,
  we quote runtimes using
  naive matrix multiplication. One can obtain further speedups using
  faster matrix multiplication algorithms with exponent \math{\omega}:
  the \math{(r\times n)}times\math{(n\times d)} product can be computed in
  time \math{O(nrd^{\omega-2})}, where \math{d\le r\le n}.
  The current best exponent is \math{\omega\approx 2.373} \cite{W2012}.}
In a typical application,
\math{d\ll r\ll n}, so we already have significant
computational gains from this simple observation.
We can further
improve the runtime by approximating \math{\matA\transp\matA} rather
than computing exactly,
\mld{
  \Atilde= \Pi_{\textsc{g},d\times r}\transp(\matA\transp\Pi_{n\times r_1}\Pi_{n\times r_1}\transp\matA)^{1/2}
  \qquad\qquad
  \text{(Figure~\ref{fig:distortion}b(\rn{1}))
  \label{eq:first}},
}%
where \math{\Pi_{n\times r_1}} can be any of the fast random projections
(Hadamard, \textsf{CountSketch}, \textsf{OSNAP}). The dimension
of the random projection \math{r_1} which is used to approximate
\math{\matA\transp\matA} can now be chosen \emph{independently} of the dimension
of the final embedding \math{r}.
In Figure~\ref{fig:distortion} we use a fast Hadamard transform
with \math{r_1=10 r} to compute the approximation
\math{\Atilde=\Pi_{\textsc{g},r\times d}
  (\matA\transp\Pi_{\textsc{h}}\Pi_{\textsc{h}}\transp\matA)^{1/2}}.
One more simple observation results in our final embedding. There is no
longer any need to have the Gaussian projection because the dimension is
already small. In fact, the additional Gaussian projection just in
\r{eq:first} just makes things worse by adding distortion.
Thus our final \math{\ell_2}-embedding is
\mld{
  \Atilde=(\matA\transp\Pi_{n\times r_1}\Pi_{n\times r_1}\transp\matA)^{1/2}
  \qquad\qquad
  \text{(Figure~\ref{fig:distortion}b(\rn{2}))
  \label{eq:second}},
}%
which is a simple, fast, nonlinear oblivious
embedding into a fixed dimension \math{d}.
Note, the significantly better distortion in 
Figure~\ref{fig:distortion}b(\rn{2}) is because we may choose
the inner-dimension \math{r_1} to optimize speed and accuracy of the
approximation to \math{\matA\transp\matA}. We can do this because the embedding
dimension for downstream machine learning is fixed at \math{d}
(the smallest possible that can  preserve norms). We have
\emph{decoupled} the efficiency of downstream machine learning (controled
by \math{d}) from
the accuracy of the embedding (controled by \math{r_1}).

\subsection{Notation}

Throughout, the target matrix \math{\matA} is a fixed
\math{n\times d} real-valued matrix, which we take to be full rank
(\math{d\ll n}).
Uppercase roman (\math{\matA,\matB,\matC,\matX\ldots}) are 
matrices, and
lowercase bold  (\math{\aa,\bb,\cc,\xx,\yy,\zz,\ldots}) are vectors.
The standard Euclidean
basis vectors are \math{\ee_1,\ee_2,\ldots} (the dimension will usually be 
clear from the context). We use the shorthand
\math{[k]} for the set \math{\{1,\ldots,k\}}.

The
singular value decomposition (SVD)
allows us to write
\math{\matA=\matU\Sigma\matV\transp}, where
the columns of \math{\matU\in\R^{n\times d}} are the
left singular vectors, the columns of 
\math{\matV\in\R^{d\times d}} are the
\math{\rho} right singular vectors, and \math{\Sigma\in\R^{d\times d}} 
is a diagonal matrix of positive singular values \math{\sigma_1\ge\cdots\ge
\sigma_d}; \math{\matU} and \math{\matV} are orthonormal,
so \math{\matU\transp\matU=\matV\transp\matV=\matI_d} \cite{GV96}. 
For integer \math{k}, we use 
\math{\matU_k\in\R^{n\times k}} (resp. \math{\matV_k\in\R^{d\times k}}) 
for the first \math{k} left (resp. right)
singular vectors, and \math{\Sigma_k\in\R^{k\times k}} is the 
diagonal matrix of corresponding top-\math{k} singular values.
We can view a matrix as a row of columns. So, 
\math{\matA=[\aa_1,\ldots,\aa_d]}, 
\math{\matU=[\uu_1,\ldots,\uu_\rho]},
\math{\matV=[\vv_1,\ldots,\vv_\rho]},
\math{\matU_k=[\uu_1,\ldots,\uu_k]} and 
\math{\matV_k=[\vv_1,\ldots,\vv_k]}. Similarly,
we may write 
\math{\matA\transp=[\xx_1,\ldots,\xx_n]}, where \math{\xx_i} is the
\math{i}th row of \math{\matA} (the data points).
We also use \math{\matA_{(i)}} and \math{\matA^{(j)}} to refer to the
\math{i}th row and \math{j}th column respectively of the matrix
\math{\matA}.

The Frobenius (Euclidean) \math{\ell_2}-norm of a matrix \math{\matA} is 
\math{\norm{\matA}_F^2=\sum_{ij}\matA_{ij}^2=\trace(\matA\transp\matA)=
\trace(\matA\matA\transp)}. The pseudo-inverse 
\math{\matA^\dagger} of \math{\matA} with 
SVD \math{\matU_\matA\Sigma_\matA\matV_\matA\transp} is  
\math{\matA^\dagger=\matV_\matA\Sigma_\matA^{-1}\matU_\matA\transp};
\math{\matA\matA^\dagger=\matU_\matA\matU_\matA\transp} is a symmetric 
projection operator.
\math{\norm{\matA}_2} is
the operator/spectral norm of \math{\matA} (top singular value), and the
\math{\ell_1}-norm of \math{\matA} is
\math{\norm{\matA}_1=\sum_{i\in[n],j\in[d]}|\matA_{ij}|
  =\sum_{i\in[n]}\norm{\matA_{(i)}}_1=\sum_{j\in[d]}\norm{\matA^{(j)}}_1}.

We use \math{c,c_1,c_2,\ldots} to generically
denote absolute constants whose values may change with each instance in which
they appear.

\section{Our Results}

We contribute two main tools: new low distortion embeddings for
\math{\ell_2} and \math{\ell_1}. The \math{\ell_1} embedding is an application
of the new \math{\ell_2} embedding, and as such can be extended
to an \math{\ell_p} embedding using the techniques in
\cite{WW2018} (we only give details for \math{\ell_1} which is the most
useful in machine learning, specifically \math{\ell_1}-regression, or
robust regression.)

\subsection{Oblivious \math{\ell_2} and \math{\ell_1} Subspace Embeddings}
To state our result for
\math{\ell_2}, we need to define an \math{\ell_2}-subspace embedding for
an orthogonal matrix, which we call an \math{\varepsilon}-JLT.
\begin{definition}[\math{\varepsilon}-JLT]
  An embedding
  matrix \math{\Pi\in\R^{n\times r}} is an \math{\varepsilon}-JLT for an orthogonal
  matrix \math{\matU\in\R^{n\times d}} if
  \mld{
    \norm{\matI-\matU\transp\Pi\Pi\transp\matU}_2\le\varepsilon.
  }
\end{definition}
Our \math{\ell_2} embedding is given by \r{eq:second}, where
\math{\Pi} is an \math{\varepsilon}-JLT for the left singular vectors of
\math{\matA}. 
  \begin{theorem}[Oblivious {\math{\ell_2}-Subspace Embedding into \math{\R^d}}]\label{theorem:L2-subspace}
  Let \math{\matA\in\R^{n\times d}} have SVD \math{\matA=\matU\Sigma\matV\transp},
  and let \math{\Pi\in\R^{n\times r}} be an \math{\varepsilon}-JLT for \math{\matU}.
  Let \math{\Atilde=(\matA\transp\Pi\Pi\transp\matA)^{1/2}}. Then,
  \math{\Atilde} is an isometry for \math{\matA}. That is,
    \mld{\hspace*{-1in}
  \text{for all \math{\xx\in\R^{d}}:}\qquad\qquad
  (1-\varepsilon)\norm{\matA\xx}^2_2\le\norm{\Atilde\xx}^2_2\le
    (1+\varepsilon)\norm{\matA\xx}^2_2.
  }
\end{theorem}
  \paragraph{Comments.}
  The \math{\ell_2}-subspace embedding is oblivious if \math{\Pi} is oblivious
  (universal, working for any \math{\matA}). It is also
  fast and simple, but nonlinear.
  The embedding dimension is~\math{d}, which is
  optimal since the rank of \math{\matA} must
  be preserved.
  The main impact of Theorem~\ref{theorem:L2-subspace} is that the
  embedding dimension is a constant, independent of
  the distortion \math{\varepsilon}. This means that \math{\varepsilon}
  can be independently optimized, either for runtime or accuracy or both. 
  Computing an accurate \math{\ell_2}-subspace embedding reduces to
  approximating
  the covariance matrix \math{\matA\transp\matA} quickly.
  The runtime has 3 terms:
  \begin{enumerate}
  \item The time to compute \math{\Pi\transp\matA}. Our main focus is dense
    matrices.
    The fast subsampled random Hadamard transform,
  \math{\Pi_\textsc{h}\matA=\text{sample}_r(\matH\matD\matA/\sqrt{r})}, with
  \math{r\ge\frac{5}{12}\varepsilon^{-2}(\sqrt{d}+\sqrt{\ln(3tn)})^2\ln d},
  is an
  oblivious construction which produces an \math{\varepsilon}-JLT for
  \math{\matU} with probability at least \math{1-1/t}
  (Lemma~\ref{lemma:SRHT-tropp}).
    Assuming \math{\ln n\le d}, the runtime is
    \math{O(nd\ln(d/\varepsilon))}. For sparse matrices, one can use
    \textsf{CountSketch}~\cite{CW2013} or \textsf{OSNAP}~\cite{NN2013}
    to get runtimes which depend
    on \math{\textsc{nnz}(\matA)} instead of \math{nd}. \textsf{CountSketch}
    and \textsf{OSNAP} are also oblivious to \math{\matA} and produce
    fast sketches which facilitate distributed and streaming environments.
  \item
    The time to compute \math{\Atilde\transp\Atilde},
    which is \math{O(rd^{\omega-1})} (\math{\omega} is the exponent
    for matrix multiplication).
  \item
    The time to compute the square-root of a \math{d\times d} matrix, which
    is \math{O(d^3)}.
  \end{enumerate}
  The total runtime to compute the subspace embedding is
  in \math{O(nd\ln(d/\varepsilon)+(d^\omega\ln d)/\varepsilon^{2}+d^3)}.
  
  Our first application of the new \math{\ell_2}-embedding is a similar tool
  for \math{\ell_1}, that is an oblivious \math{\ell_1}-subspace embedding.
  We use the \math{\ell_2}-subspace embedding as a black-box for obtaining an
  \math{\ell_1}-subspace embedding using the construction in
  \cite{WW2018}. Fix a parameter \math{t>0} which controls the
  failure probability.
  Let \math{S_1,\ldots,S_r} be random partition of
  \math{[n]} into \math{r} bins, where
  \math{r\sim d\ln d.}
  Define the subspace embedding
  \math{\Pi_2\transp\matA\in\R^{r\times d}} as follows: rescale each
  row \math{\matA_{(i)}} by an
  independent Cauchy random variable \math{\rvC_i} and add the rescaled rows in
  \math{S_j} to get the \math{j}th row of \math{\Pi_2\transp\matA},
  \mld{
    [\Pi_2\transp\matA]_{(j)}=\sum_{i\in S_j}\rvC_i\cdot\matA_{(i)}=
    \sum_{i\in S_j}\rvC_i\cdot\ee_i\transp\matA.
    \label{eq:intro-pi2}
  }
  The embedding \math{\Pi_2\transp\matA} can be computed in time
  \math{O(\textsc{nnz}(\matA))}.
  The \math{\ell_1}-subspace embedding is obtained by concatenating our
  \math{\ell_2}-embedding with \math{\Pi_2\transp\matA}:
  \mld{
  \Atilde=
  \left[
  \begin{matrix}
    \sqrt{d}\ln(td)\cdot\Atilde_1\\
    \Atilde_2
  \end{matrix}
  \right]
  =
  \left[
    \begin{matrix}
    \sqrt{d}\ln (td)\cdot(\matA\Pi_1\Pi_1\transp\matA)^{1/2}\\
    \Pi_2\transp\matA
    \end{matrix}
    \right],
  \label{eq:embedding-def-intro}
}
where \math{\Atilde_1} is our \math{\ell_2}-embedding and \math{\Atilde_2} is
the \math{\ell_1} part of the embedding from \cite{WW2018}. The dimension
of the embedding is \math{d+r}. Our main tool for \math{\ell_1}-subspace
embedding is Theorem~\ref{theorem:main-embedding} which we summarize here.
\begin{theorem}[Oblivious {\math{\ell_1}-Subspace Embedding into \math{\R^{d+r}}}]\label{theorem:L1-subspace}
  Let \math{\matA\in\R^{n\times d}} have SVD \math{\matA=\matU\Sigma\matV\transp},
  and let \math{\Pi_1\in\R^{n\times r_1}} be a \math{\frac12}-JLT for
  \math{\matU}.
  Let \math{\Atilde_1=(\matA\transp\Pi_1\Pi_1\transp\matA)^{1/2}} and let
  \math{\Atilde_2=\Pi_2\transp\matA} as given in \r{eq:intro-pi2}.
  Construct \math{\Atilde} from \math{\Atilde_1} and \math{\Atilde_2} as
  in \r{eq:embedding-def-intro}. Then, with constant probability,
  for all \math{\xx\in\R^{d}},
  \begin{enumerate}[label={(\roman*)}]
  \item Setting  \math{r\in O(d\ln d)} gives a distortion
    \math{O(d\ln d)}:
      \mld{
    \Omega(1)\cdot \norm{\matA\xx}_1\le\norm{\Atilde\xx}_1\le
    O(d\ln d)\cdot \norm{\matA\xx}_1.\label{eq:l1-embedding-1}
  }
    \item Setting \math{r\in O(d\ln^{1+\eta} d)} for any \math{0<\eta\le\frac13}
      gives distortion \math{O((d\ln d)/\ln\ln d)}:
      \mld{
    \Omega(\ln\ln d)\cdot \norm{\matA\xx}_1\le\norm{\Atilde\xx}_1\le
    O(d\ln d) \cdot \norm{\matA\xx}_1.\label{eq:l1-embedding-2}
    }
  \end{enumerate}
\end{theorem}
\paragraph{Comments.}
The embedding is oblivious as long as \math{\Pi_1} in
\r{eq:embedding-def-intro} is oblivious.
We compare the most recent results in~\cite{WW2018} with
Theorem~\ref{theorem:L1-subspace}. Our construction is
the same as in~\cite{WW2018}, but using our \math{\ell_2}-embedding
which influences the choice of \math{r}. The \math{\ell_1}-embedding
in~\cite{WW2018} is based on the
\textsf{CountSketch} and \textsf{OSNAP} \math{\ell_2}-embeddings.
The \textsf{CountSketch} approach in~\cite{WW2018} embeds into
\math{r=O(d^2)+O(d\log^2d)} dimensions (the first term is from the
\math{\ell_2}-\textsf{CountSketch}-part and the second term is from the
\math{\ell_1} part). The distortion achieved is \math{O(d)}.
The \textsf{OSNAP} approach in~\cite{WW2018} embeds into
\math{r=O(B\cdot d\ln d)} dimensions with distortion
\math{O(d\log_B d)}, where \math{B} is a parameter controling
the accuracy and runtime of the \math{\ell_2}-embedding.
For constant \math{B}, the asymptotic behavior of the
\textsf{OSNAP} approach is comparable to our algorithm
(part (\rn{1}) in Theorem~\ref{theorem:L1-subspace}).
Setting \math{B=\ln d} in the \textsf{OSNAP} approach
gives embedding dimension \math{O(d\ln^2d)} with distortion
\math{O((d\ln d)/\ln\ln d)}, which achieves comparable
distortion to our result
(part (\rn{2}) in Theorem~\ref{theorem:L1-subspace}), but with a
\math{(\ln d)}-factor increase in the embedding dimension.
In summary, 
\begin{enumerate}[label={(\arabic*)}]
\item Our embedding dimension is slightly tighter which can be
  important for applications.
\item The dimension and accuracy of the overall
  \math{\ell_1}-embeding are
  separated from the runtime and dimension of the \math{\ell_2} part of the
  embedding.
  Therefore, any good \math{\ell_2}-embedding can be used.
\end{enumerate}
Regarding point (2) above, our \math{\ell_1}-embedding trully
uses \emph{any} \math{\ell_2} embedding as a black-box.
Whereas, the \textsf{CountSketch} and \textsf{OSNAP} approaches are specific
to those embeddings because those \math{\ell_2}-embeddings, by
construction, preserve \math{\ell_1}-dilation.
For a typical black-box \math{\ell_2}-embedding, this may not be the case,
hence, we need a different proof
to accommodate an arbitrary
\math{\ell_2}-embedding. Though our proof accomodates an arbitrary
\math{\ell_2} embedding, the distortion we achieve depends on the
\math{\ell_2}-embedding dimension, so \r{eq:l1-embedding-1} and
\r{eq:l1-embedding-2} apply for fixed
\math{\ell_2}-embedding dimension \math{d}.

\subsection{\math{\ell_2}-Applications}

Our nonlinear embedding can be used in the standard
\math{\ell_2} applications to obtain near-optimal (in relative error)
low-rank approximation, regularized linear regression and leverage scores. 
Since our embeding is nonlinear, we need to adapt and or re-analyze the 
existing
algorithms which use a linear embedding to quickly approximate the
corresponding tasks.

\subsubsection{Regularized \math{\ell_2}-Regression}

Given \math{\matA\in\R^{n\times d}}
and \math{\bb\in\R^{n\times 1}}, find \math{\xx_*} which minimizes the
regularized \math{\ell_2}-reconstruction error over some domain
\math{\cl D}:
\mld{\xx_*=\mathop{\arg\min}_{\xx\in\cl D}\ \norm{\matA \xx-\bb}_2+\lambda\Phi(\xx).}
The domain \math{\cl D} can be arbitrary, and the regularizer \math{\Phi} can be an abitrary nonnegative function.
The standard approach with a linear embedding \math{\Pi\transp}
is to minimize
\math{\norm{\Pi\transp(\matA\xx-\bb)}_2+\lambda\Phi(\xx)} over \math{\cl D}.
This cannot be implemented with our nonlinear embedding. Instead,
We construct \math{\matX=[\matA,-\bb]} and let
\math{\tilde\matX} be our \math{\ell_2}-subspace embedding for
\math{\matX}.
So, with constant probability, for all \math{\zz\in\R^{d+1}},
\mld{
  (1-\varepsilon)\norm{\matX\zz}_2^2\le\norm{\tilde\matX\zz}_2^2\le 
  (1+\varepsilon)\norm{\matX\zz}_2^2.\label{eq:isometry-X}
}
Write \math{\tilde\matX=[\Atilde,-\tilde\bb]}, and construct
\math{\tilde\xx} by solving the regression problem with
\math{\Atilde, \tilde\bb}:
\mld{\tilde\xx=\mathop{\arg\min}_{\xx\in\cl D}\ \norm{\Atilde \xx-\tilde\bb}_2+\lambda\Phi(\xx).\label{eq:intro-approx-regression}
}
The following theorem follows from a standard
sandwich argument:
\begin{theorem}[Relative-Error Constrained, Regularized  \math{\ell_2}-Regression]
  \label{theorem:L2-regression-intro}
  Construct \math{\Atilde, \tilde\bb} as described above
  from the embedding of \math{\tilde\matX=[\matA,-b]} and
  let \math{\tilde\xx} solve
  \r{eq:intro-approx-regression}. Then, with constant probability,
  for all \math{\xx\in\cl D},
  \mld{\norm{\matA\tilde\xx-\bb}_2^2+\Phi(\tilde\xx)\le \left(\frac{1+\varepsilon}{1-\varepsilon}\right)\cdot \left(\norm{\matA \xx-\bb}_2^2+\Phi(\xx)\right).}
\end{theorem}
\paragraph{Comments.}
One can extend~\ref{theorem:L2-regression-intro} to the case where
\math{\bb} becomes a matrix \math{\matB\in\R^{n\times q}}, which can be
used to solve the subspace
reconstruction problem.
Theorem~\ref{theorem:L2-regression-intro}
applies to any constraint \math{\cl D} and
any regularizer~\math{\Phi}.
The success probability can be boosted to \math{1-\delta} by independently
repeating the algorithm \math{O(\log(1/\delta))} times and picking the
best solution.
Theorem~\ref{theorem:L2-regression-intro}
does not give an algorithm for solving the
constrained, regularized \math{\ell_2}-regression. Rather it shows how
to \emph{quickly} reduces the
problem to one with a fixed number of \math{d} ``data points''
while getting \math{O(1+\varepsilon)}-relative error accuracy.
Solving the smaller problem should be
considerably more efficient, depending on the nature of the constraints and
regularizer.
The runtime for simple unconstrained regression with dense matrices
is the time to compute
the embedding of \math{\matX} plus \math{O(d^3)} for the small
\math{d\times d} regression, for a total runtime
of
\math{O(nd\ln(d/\varepsilon)+(d^\omega\ln d)/\varepsilon^{2}+d^3)}.
Our algorithm uses \math{(\Atilde\transp\Atilde)^{1/2}} which has a similar
condition number to \math{\matA}.
Using a standard linear embedding approach,
one can achieve a similar runtime by computing
\math{\tilde\xx=(\Atilde\transp\Atilde)^{-1}\Atilde\transp\tilde\bb}, where
\math{\Atilde=\Pi\transp\matA} and \math{\tilde\bb=\Pi\transp\bb}. However,
this is not the method of choice in practice because
the condition number of \math{\Atilde\transp\Atilde} is that of
\math{\Atilde} squared. Hence computing
\math{(\Atilde\transp\Atilde)^{-1}} is succeptible to numerical instability,
and one usually solves the regression problem using a QR-factorization of
\math{\Atilde}, in which case the total runtime becomes
\math{O(nd\ln(d/\varepsilon)+(d^3\ln d)/\varepsilon^{2}+d^3)}.\footnote{There
  are faster ways to perform a QR with runtime \math{o(rd^2)}~\cite{K1995},
  but such
  algorithms are not mainstream, and still slower than matrix multiplication.}
Our approach saves by doing multiplication with the
``large'' \math{(\varepsilon^{-2}d\ln d)\times d} matrix as opposed to a QR.
More complex regressions with complicated constraints
\math{\cl D} and inconvenient regularizers (such as
\math{\ell_0} or \math{\ell_1} regularization) will only
further highlight the
computational benefits of our fixed dimension embedding.

\subsubsection{Low Rank Approximation (PCA)}

Let \math{\Atilde} be our subspace embedding for \math{\matA} satisfying
\r{eq:isometry}, and 
let \math{\Atilde=\tilde\matU\tilde\Sigma\tilde\matV\transp}
be its SVD. Let \math{\tilde\matV_k}
be the top-\math{k} right singular vectors of
\math{\Atilde} (the first \math{k} columns of \math{\tilde\matV_k}).
So,
\math{\Atilde_k=\Atilde\tilde\matV_k\tilde\matV_k} is the best rank-\math{k}
approximation to \math{\Atilde}. We treat \math{\tilde\matV_k} as an
approximate top-\math{k} PCA of \math{\matA} and construct 
\math{\hat\matA_k=\matA\tilde\matV_k\tilde\matV_k} as a
rank-\math{k} approximation to \math{\matA} (\math{\matA_k} is
the best rank-\math{k} approximation to \math{\matA}).
Theorem~\ref{theorem:L2-PCA-intro} gives the quality of approximation of our
approximate PCA.
\begin{theorem}[Relative Error Low Rank Approximation (PCA)]
  \label{theorem:L2-PCA-intro}
  Suppose \math{\Atilde} satisfies \r{eq:isometry}. Then,
  for \math{k\in[d]}, the matrix \math{\hat\matA_k} as constructed above
    satisfies
    \mld{\norm{\matA-\hat\matA_k}_2^2\le\left(\frac{1+\varepsilon}{1-\varepsilon}\right)\cdot
      \norm{\matA-\matA_k}_2^2.}
\end{theorem}
\paragraph{Comments.}
The running time to compute \math{\tilde\matV_k} is the time to compute
\math{\Atilde} and its SVD. For dense matrices this is
\math{O(nd\ln(d/\varepsilon)+(d^\omega\ln d)/\varepsilon^{2}+d^3)}.
To compute the reconstruction
\math{\hat\matA_k=\matA\tilde\matV_k\tilde\matV_k\transp}, takes an additional
\math{O(ndk)} time. The algorithm above, which is based
on an embedding \math{\Atilde} satisfying \r{eq:isometry}, is standard.
The benefits of
this approach come purely from the fixed embedding dimension \math{d}.

\subsubsection{\math{\ell_2}-Leverage Scores}
The \math{\ell_2}-leverage scores are the diagonal entries of the projection
operator \math{\matA\matA^\dagger}. The leverage scores have statistical
significance and play an important
role in sampling based algorithms~\cite{malik186}. For
\math{i\in[n]}, the \math{i}-th leverage score is
\mld{\tau_i=
  \norm{\ee_i\transp\matA\matA^\dagger}_2^2.}
We approximate
\math{\matA\matA^\dagger\approx\matA\Atilde^\dagger=\matA\Atilde^{-1}} (because \math{\Atilde} is square), giving the 
leverage score estimates
\mld{
  \tilde\tau_i=
  \norm{\ee_i\transp\matA\Atilde^{-1}}_2^2.\label{eq:intro-lev}}
We prove that our estimate gives
a relative error approximation to the true leverage score:
\begin{theorem}[Relative-Error Leverage Scores]
  \label{theorem:L2-leverage-intro}
  For \math{i\in[n]},
  the leverage scores \math{\tilde\tau_i} in \r{eq:intro-lev} constructed
  using the embedding \math{\Atilde} satisfying \r{eq:isometry}
  from Theorem~\ref{theorem:L2-subspace}
  are a relative
  error approximation to the true leverage scores \math{\tau_i},
  \mld{|\tilde\tau_i-\tau_i|\le \frac{\varepsilon}{1-\varepsilon}\cdot
    \tau_i.}
\end{theorem}
\paragraph{Comments.}
Computing \math{\Atilde^{-1}} takes
\math{O(nd\ln(d/\varepsilon)+(d^\omega\ln d)/\varepsilon^{2}+d^3)}
time. However, the product \math{\matA\Atilde^{-1}} takes an additional
time \math{O(nd^{\omega-1})}, which is expensive. This runtime
can be improved while maintaining the relative error approximation by
observing that we only need the norms of the rows of
\math{\matA\Atilde^{-1}}. Therefore, we can apply a norm-preserving
JLT to these rows and compute the row-norms of
\math{\matA\Atilde^{-1}\Pi_{\textsc{g}}} where \math{\Pi_{\textsc{g}}} is
a \math{d\times O(\varepsilon^{-2}\ln n)} matrix of random Gaussians
(this standard trick was developed in \cite{malik186}).
Now, the time to approximate the row-norms using the product 
\math{\matA\Atilde^{-1}\Pi_{\textsc{g}}} is
\math{O(\varepsilon^{-2}d^2\ln n+\varepsilon^{-2}nd\ln n)}. For most practical
applications, only a constant factor approximation to the
leverage scores is required, so our runtime is
\math{O(nd\ln n+d^\omega\ln n+d^3)}, where \math{\omega} is the exponent
for whatever matrix multiplication algorithm is used (this
beats the runtime of
\math{O(nd\ln n+d^3\ln n\ln d)} from~\cite{malik186}).

\subsection{\math{\ell_1}-Applications}

We consider some applications of subspace embeddings to
\math{\ell_1}-leverage scores and coresets for \math{\ell_1} regression.
We also briefly discuss the improvements to other
\math{\ell_1} variants of regression (distributed and streaming models),
entrywise low-rank approximation and quantile regression
(see \cite[Section 1.3]{WW2018}).

\subsubsection{Regularized \math{\ell_1}-Regression}

The approach for \math{\ell_1}-regression is the same as for
\math{\ell_2}-regression. The
idea is to construct a non-oblivious \math{(1+\varepsilon)}-relative
error \math{\ell_1}-embedding
which then allows us to prove a result analogous to
Theorem~\ref{theorem:L2-regression-intro} for \math{\ell_1},
via a the same sandwiching argument. The proof
is exactly analogous to the proof of
Theorem~\ref{theorem:L2-regression-intro} using
\math{\norm{\cdot}_1} instead of \math{\norm{\cdot}_2^2}, so
we simply state the result.
We are given \math{\matA\in\R^{n\times d}}
and \math{\bb\in\R^{n\times 1}} and a domain for the optimization
\math{\cl D\subseteq\R^{d}}.
\begin{theorem}[Relative-Error Constrained, Regularized \math{\ell_1}-Regression]
  \label{theorem:L1-regression-intro}
  Let \math{\matX=[\matA,-\bb]\in\R^{n\times (d+1)}} and let
\math{\tilde\matX\in\R^{r\times (d+1)}} be an \math{\ell_1}-subspace embedding for
\math{\matX} satisfying, for all \math{\zz\in\R^{d+1}},
\mld{
  (1-\varepsilon)\norm{\matX\zz}_1\le\norm{\tilde\matX\zz}_1\le 
  (1+\varepsilon)\norm{\matX\zz}_1.\label{eq:isometry-X-L1}
}
Write \math{\tilde\matX=[\Atilde,-\tilde\bb]}, and construct
\math{\tilde\xx} by solving the \math{\ell_1-regression} problem with
\math{\Atilde, \tilde\bb}:
\mld{\tilde\xx=\mathop{\arg\min}_{\xx\in\cl D}\ \norm{\Atilde \xx-\tilde\bb}_2+\lambda\Phi(\xx).\label{eq:intro-approx-regression-L1}
}
Then, for all \math{\xx\in\cl D},
\mld{\norm{\matA\tilde\xx-\bb}_1+\Phi(\tilde\xx)\le \left(\frac{1+\varepsilon}{1-\varepsilon}\right)\cdot \left(\norm{\matA \xx-\bb}_2^2+\Phi(\xx)\right).}
\end{theorem}
\paragraph{Comments.}
Many of the comments after
Theorem~\ref{theorem:L2-regression-intro} apply to
Theorem~\ref{theorem:L1-regression-intro}.
One main difference between \math{\ell_1} and \math{\ell_2} is that the
oblivious embedding does not give a relative error embedding in \math{\ell_1}.
In fact, this is not possible given the lower bounds in~\cite{WW2018}.
Therefore, to apply
Theorem~\ref{theorem:L1-regression-intro}, some form of
non-oblivious embedding satisfying
\r{theorem:L1-regression-intro} is needed.
So, unlike the \math{\ell_2} case, in the embedding dimension \math{r} in
Theorem~\ref{theorem:L1-regression-intro} may not be a constant. Nevertheless,
there are significant gains when \math{r} depends on \math{d},
not \math{n}.

\subsubsection{Relative Error \math{\ell_1}-Embedding Via
  Well Conditioned Bases and Lewis Weights}

We compare 3 types of algorithms for constructing a
relative error \math{\ell_1}-embedding for use in
Theorem~\ref{theorem:L1-regression-intro}.
All the algorithms 
all based on
constructing a coreset of rescaled rows from the matrix by
row-sampling using special probabilities which can loosely be referred to as
\math{\ell_1}-leverage scores.
\begin{enumerate}
\item
  A straightforward well-conditioned basis approach that
  directly uses the oblivious embedding in \r{eq:embedding-def-intro}.
  We get a coreset size (embedding dimension) of \math{O(d^{3.5}\log^{1.5} d)}.
\item
  A well-conditioned basis approach which
constructs a basis using a two step approach, using the
oblivious embedding to get initial probabilities and then
ellipsoidal rounding to get a better conditioned basis (the direct
ellipsoidal rounding approach as in \cite{DDHKM2009} is too slow).
We get a coreset size
of  \math{O(d^{2.5})}, which is a better coreset size
at the expense of a \math{\text{poly}(d)} additive increase in runtime.
\item Sampling with Lewis \math{\ell_1}-weights which construct
  \math{\ell_1} leverage scores by converting to \math{\ell_2}, \cite{CP2015}.
  Lewis \math{\ell_1}-weights are constructed via iterative
  calls to a black-box fast approximation algorithm for
  \math{\ell_2}-leverage scores for which we use
  \r{eq:intro-lev} with \math{\Atilde} given in \r{eq:second}.
  This gives the best coreset size of \math{O(d\ln d)} but requires
  \math{O(\log_2\log_2n)} passes through the data to compute
  approximations to the Lewis weights.
\end{enumerate}
First, we give an \math{\ell_1}-sampling lemma which follows directly from the
methods in \cite{malik186}. For completeness, we give a proof of this
theorem in Section~\ref{proof:l1-sampling},
which mostly follows the same line of reasoning
as in \cite{malik186}.
\begin{theorem}[Non-Oblivious Embedding Via \math{\ell_1}-Sampling]
  \label{theorem:l1-sampling}
  Let \math{\matU} be a basis for the range of \math{\matA}. Define the
  \math{\ell_1}-condition number of \math{\matU} by\footnote{\math{\alpha(\matU)} equals \math{\alpha\beta} from the definition of a
    well conditioned basis in~\cite{DDHKM2009}.}
  \mld{\alpha(\matU)=\frac{\norm{\matU}_1}{\min_{\norm{\xx}_\infty=1}{\norm{\matU\xx}_1}}.\label{eq:def-alp}}
  and let
  \math{\lambda_i=\text{median}(\matU_{(i)}\matC)}, where
  \math{\matC} is a \math{d\times(15\ln (6n/\delta))}
  matrix of independent Cauchys.
  For \math{s\ge 28\varepsilon^{-2}d\alpha(\ln(18/\varepsilon)+
    d^{-1}\ln(3/\delta))}, define sampling
  probabilities 
  \mld{
    p_i=\min\left(s\cdot\frac{\lambda_i}{\sum_{i\in[n]}\lambda_i},1\right).}
  Independently (for each row \math{i} of \math{\matA}), with probability
  \math{p_i}, add
  the (rescaled) row \math{\matA_{(i)}/p_i} as a row of
  \math{\Atilde}. With probability at least \math{1-\delta-e^{-s/3}}, 
  for all \math{\xx\in\R^d},
    \mld{
      (1-\varepsilon)\norm{\matA\xx}_1\le\norm{\Atilde\xx}_1\le 
      (1+\varepsilon)\norm{\matA\xx}_1,\label{eq:thm-L1-sampling-isometry}}
    and the embedding dimension (number of rows in \math{\Atilde})
    is at most \math{2s}.
\end{theorem}
\paragraph{Comments.}
Theorem~\ref{theorem:l1-sampling}
says that a good basis \math{\matU} can be used to get
a relative error embedding.
The parameter \math{\alpha} determines how good, i.e. well conditioned,
the basis
\math{\matU}
is. Theorem~\ref{theorem:l1-sampling} actually
does more than just embedd, it produces a coreset formed from the
rows of \math{\matA}, suitably rescaled. The coreset size is
\math{O(\alpha d/\varepsilon^2)}.
There are several ways to construct a basis \math{\matU}
with \math{\alpha(\matU)\le \text{poly}(d)}, which we discuss in the next lemma.
\begin{lemma}\label{lemma:constructing-U}
  The following are three methods to construct a well  conditioned basis of the
  form
  \mld{\matU=\matA\matR^{-1}.}
  \begin{enumerate}[label={(\roman*)}]
    \item {[Using the \math{\ell_2}-Embedding.]}
  Let \math{\Atilde} be the \math{\ell_2}-subspace embedding from
  Theorem~\ref{theorem:L2-subspace} with \math{\varepsilon=\frac12}, and
  let \math{\matR=\Atilde}, and \math{\matU=\matA\Atilde^{-1}}.
  Then, with constant probability,
  \math{\alpha\le \sqrt{3}d^2\kappa_1(\matA)},
  where \math{\kappa_1(\cdot)} is an \math{\ell_2}
  condition number of \math{\matA} relative
  to an optimal Auerbach basis,\footnote{An Auerbach basis has \math{\alpha\le d} and always exists, \cite{A1930}.
    The condition number \math{\kappa_1(\matA)} 
is analogous to the standard condition number \math{\kappa_2(\matA)}
where \math{\matQ} is chosen as an orthogonal
  basis for the range of \math{\matA}, in which case
  \math{\kappa_2(\matA)=\norm{\matS}_2\norm{\matA^{-1}}_2=\norm{\matA}_2\norm{\matA^{-1}}_2}.
}
  \mld{
  \kappa_1(\matA)=\norm{\matA^{-1}}_2
  \min_{
    {\matU:\matA=\matQ\matS};\atop
    {\alpha(\matQ)\le d}
  }
  {\norm{\matS}_2}.\label{eq:def:kappa1}}
  In Theorem~\ref{theorem:l1-sampling}
  the coreset size is \math{O(d^3\kappa_1)} and the total runtime to embed is
  in \math{O(nd\ln n+d^3)}.
\item{[Using the \math{\ell_1}-Embedding.]}
  Let \math{\Atilde\in\R^{r\times d}} be an \math{\ell_1} embedding satisfying for
  all \math{\xx\in\R^d}
  \mld{
    \norm{\matA\xx}_1\le\norm{\Atilde\xx}_1\le
    \kappa\cdot \norm{\matA\xx}_1.
    \label{eq:getU-part2-isometry}
  }
  Use a QR factorization of \math{\Atilde} to compute \math{\matR},
  \math{\Atilde=\matQ\matR}, so
  \math{\matU=\matA\matR^{-1}=\matA(\matQ\transp\Atilde)^{-1}}.
  Then, \math{\alpha\le\kappa d\sqrt{r}}, where
  \math{\kappa, r} are given in 
  Theorem~\ref{theorem:L1-subspace}.
  In Theorem~\ref{theorem:l1-sampling}
  the coreset size is \math{O(d^{3.5}\ln^{1.5}d)} and the total runtime to embed is
  in \math{O(nd\ln n+d^3\ln d)}.
\item{[\math{\ell_1}-Embedding plus Ellipsoidal Rounding]}.
  Using Theorem~\ref{theorem:l1-sampling} with (\rn{2}) above, construct
  a constant distortion embedding
  \math{\Atilde} in \math{O(nd\ln n+d^3\ln d)} time, with 
  embedding dimension \math{r\in O(d^{3.5}\ln^{1.5}d)}.
  Let \math{\Atilde=\tilde\matQ\matS} and as in Theorem~4 of
  \cite{DDHKM2009}, construct the John ellipsoid for
  \math{\tilde\matQ} characterized by the positive definite form
  \math{\matG\transp\matG}, where \math{\matG\in\R^{d\times d}}.
  Set \math{\matR=\matG\matS}, so \math{\matU=\matA\matS^{-1}\matG^{-1}}.
  The runtime to compute
  \math{\matR} is \math{O(d^{8.5}\ln^{1.5}d)} and \math{\alpha\in O(d^{1.5})}.
  In Theorem~\ref{theorem:l1-sampling}
  the coreset size is \math{O(d^{2.5})} and the total runtime to embed is
  in \math{O(nd\ln n+d^{8.5}\ln^{2.5} d)}.  
  \end{enumerate}
\end{lemma}
\paragraph{Comments.}
In Lemma~\ref{lemma:constructing-U}, part (\rn{1})
is useful for matrices which are
well conditioned according to \math{\kappa_1}.
In practice \math{\kappa_1} is often
small, for example if the entries of \math{\matA} are independent, mean zero,
then \math{\kappa_1(\matA)\in1+O(\sqrt{d/n})}. 
When using Theorem~\ref{theorem:l1-sampling} with
Lemma~\ref{lemma:constructing-U},
one must construct the
product \math{\matA\matR^{-1}\matC} and then the median entry in each row
\math{i} to get the sampling probabilities \math{\lambda_i}.
To achieve the runtimes as claimed, this product \math{\matA\matR^{-1}\matC}
must be computed from right to left.
In part (\rn{2}) of the lemma,
the distortion of the embedding, \math{\kappa},
and the embedding dimension \math{r} are both
important in determining the conditioning parameter~\math{\alpha},
hence there is value in optimizing both \math{\kappa} and
\math{r}. In part (\rn{3}) of the Lemma, we use the same method as
in~\cite{DDHKM2009}, but avoid the \math{O(nd^5\ln n)} runtime by computing
the John ellipsoid in the smaller subspace. More efficient and/or
approximate algorithms for the John quadratic form can be useful to
reduce the \math{O(d^{8.5}\ln d)} portion of the runtime.

We now discuss the approach based on sampling probabilities
defined as the Lewis weights,~\cite{CP2015}.
This
approach needs a fast approximation algorithm for statistical
\math{\ell_2} leverage scores, which is where we use our \math{\ell_2}
embedding with Theorem~\ref{theorem:L2-leverage-intro}.
The sampling probabilities \math{w_i} (Lewis weights), 
are computed via a fixed point iteration that  solves
\mld{w_i=\tau_i(\matW^{-1/2}\matA),}
where \math{\matW} is a diagonal matrix whose diagonal entries are
\math{\matW_{ii}=w_i} and \math{\tau_i(\matW^{-1/2}\matA)}
are the \math{\ell_2} of \math{\matW^{-1/2}\matA}.
The fixed point iteration to compute \math{w_i} is given below
\begin{center}
  \fbox{\parbox{0.95\textwidth}{
\begin{algorithmic}[1]
  \STATE Initialize \math{w_i=1} for \math{i\in[n]}. Let \math{\matW} be the
  diagonal matrix with diagonal entries \math{\matW_{ii}=w_i}.
\FOR{\math{i=1,\ldots,T}}
\STATE Approximate the leverage scores \math{\tau_i}  of \math{W^{-1/2}\matA},
for
\math{i\in[n]}, with \math{\tilde\tau_i}:
\begin{enumerate}[labelwidth=100pt,label={(\alph*)}]
\item Let \math{\matX=W^{-1/2}\matA} and compute the matrix product
  \math{\matU=\matX{\tilde\matX}^{-1}\mat{G}} from right to left, where
  \mandc{
    \begin{array}{l}
      \text{\math{\mat{G}} is a \math{d\times (c\ln n)} matrix of independent
        standard Gaussians which is an (\math{\varepsilon=\frac14})-JLT;}\\
      \text{\math{\tilde\matX=(\matX\transp\Pi\Pi\transp\matX)^{1/2}} is our
        \math{\ell_2}-embedding of \math{\matX} from Theorem~\ref{theorem:L2-subspace} with
      \math{\varepsilon=\frac14}.}      
    \end{array}
  }
\item \math{\tilde\tau_i\gets\norm{\matU_{(i)}}_2^2}. 
\end{enumerate}
\STATE Update the Lewis weights:
\math{w_i\gets \sqrt{w_i\tau_i}} for \math{i\in[n]}.
\ENDFOR
\RETURN \math{w_i} for \math{i\in[n]}.
  \end{algorithmic}
}}
\end{center}
\paragraph{Comments.}
In step 3(a), choosing the dimensions of
\math{\mat{G}} as \math{d\times(300\ln n)} is sufficient to
give a \math{\frac14}-JLT with probability at least
\math{1-1/n}. The choice 
\math{\varepsilon=\frac14} in step 3(a) 
gives a 6-factor approximation to the
leverage scores, with an appropriate probability. 
The approximation  must hold for every iteration, that is one must take
a union bound over the failure probabilities in each of the
\math{T} iterations. We need \math{T\ge2\log_2\log_2 n},
which does not asymptotically affect the dimension of \math{\Pi} in the
\math{\ell_2}-embedding (see the comments after
Theorem~\ref{theorem:L2-subspace}). The runtime of the entire algorithm to
compute approximate Lewis weights is in
\math{O(T\cdot(nd\ln n+d^\omega\ln n+d^3))}.

Conditioning on the leverage scores being
a 6-factor approximation in each iteration, the final Lewis weight
approximations after \math{T\ge 2\log_2\log_2 n} are approximately a
6-factor approximation
to the true Lewis weights.
By sampling and rescaling rows of \math{\matA} independently and with
replacement using probabilities defined by the approximate Lewis weights,
we get an \math{\ell_1}-subspace embedding as summarized in the
next~Lemma~\ref{lemma:lewis-intro}, which is essentially Theorem
2.3 of \cite{CP2015} using Theorem~\ref{theorem:L2-leverage-intro}
for fast approximation of leverage scores. 
\begin{lemma}[Lewis Weights for \math{\ell_1}-Embedding,
    {\cite[Theorem 2.3]{CP2015}}, using
    Theorem~\ref{theorem:L2-leverage-intro} for Leverage Scores]
  \label{lemma:lewis-intro}
  Approximate 
  the Lewis weights by \math{w_i} which are
  output from the algorithm above with
  \math{T=2\log_2\log_2 n} iterations of updating.
  Use sampling probabilities \math{p_i=w_i/\sum_{j\in[n]}w_j} and
  construct a sampling matrix \math{\Pi\transp} with
  \math{r\approx 72c\varepsilon^{-2} d\ln(72c\varepsilon^{-2} d)} 
  rows, where each row of \math{\Pi\transp} is chosen independently to be 
  \math{\ee_i\transp/p_i} with probability \math{p_i}. Let
  \math{\Atilde=\Pi\transp\matA}. Then, for all \math{\xx\in\R^d},
  \mld{
    (1+\varepsilon)^{-1}\norm{\matA\xx}\le \norm{\Atilde\xx}\le
      (1+\varepsilon)\norm{\matA\xx}.
  }
\end{lemma}
\paragraph{Comments.}
The constant \math{c} appearing in the theorem is an absolute constant
(\math{C_s} in Theorem 2.3 of~\cite{CP2015}). The coreset size is
\math{O(\varepsilon^{-2}d\ln(d/\varepsilon))}. The theoretical
coreset size is smallest, but the algorithm is
relatively more complicated, requiring \math{\log_2\log_2 n} passes
through the data to estimate leverage scores of row-rescaled versions
of \math{\matA}. 

\subsubsection{\math{\ell_1}-Regression in the Distributed Model}

In the distributed setting, the rows of matrix \math{\matA} are stored on \math{k} machines that are coordinated by
a central server. The \math{\ell_1}-regression problem is considered from this perspective in~\cite{WZ2013}, with the goal of minimizing the
amount of communication required to solve the problem to relative accuracy. The algorithm proposed in~\cite{WZ2013} to solve
this problem uses several rounds of distributed sketching to compute a final \math{O(d^3)} coreset of rows of \math{\matA} that are then
collected onto one machine and used to locally solve an \math{\ell_1} problem whose solution is a \math{(1+\varepsilon)} 
accurate approximate solution to the full \math{\ell_1} problem. 

The total communication cost of the original algorithm is 
\mld{O(kd^{2 + \eta} + d^5 \ln d + d^{4} \ln(1/\varepsilon)/\varepsilon^2),}
where \math{\eta > 0} is arbitarily small. All three terms in the cost are determined by the choice of \math{\ell_1}-embeddings. The authors of~\cite{WW2018}
observe that using the \math{\ell_1}-embeddings constructed in that work leads to a lower communication cost of 
\mld{O(kd^2 \ln d + d^{9/2} \ln^{3/2} d + d^4 \ln(1/\varepsilon)/\varepsilon^2).}
Substituting the first of the \math{\ell_1}-embeddings presented in Theorem~\ref{theorem:L1-subspace} leads to this same
improved communication cost, and using the second embedding results in a communication cost of 
\mld{O(k d^2 \ln^{1 + \eta} d + (d^{9/2} \ln^{3/2}d)/\ln\ln d + d^4 \ln(1/\varepsilon)/\varepsilon^2 ),}
which has a slightly smaller dominant term,
where $\eta \in (0, 1/3]$ is arbitrary.


\subsubsection{Quantile Regression}
Quantile regression is a robust alternative to least squares regression: the
latter models the conditional mean of the dependent variable, while the former
models the quantiles of the dependent variable. Let \math{\tau \in (0, 1)} indicate
the desired quantile, then given the design matrix
\math{\matA\in\R^{n\times d}} and response vector
\math{\bb\in \R^{n \times 1}}, the regression coefficients that
model the \math{\tau}th quantile are 
\mld{\xx_* = \mathop{\arg\min}_{\xx\in \R^{d}} \rho_{\tau}(\matA \xx - \bb),}
where \math{\rho_{\tau}(\xx) = \sum_{i=1}^d \rho_{\tau}(x_i)}, where
\mld{\rho_{\tau}(z) = \begin{cases}\tau z & z \geq 0 \\ (\tau - 1) z & z < 0 \end{cases}}
is a tilted absolute value function. Note that when $\tau=1/2$, the median regression problem is
an $\ell_1$ regression problem.  

The authors of~\cite{YMM2014} present an algorithm for the case of arbitrary $\tau \in (0,1)$ that reduces
the \math{n \times d} quantile regression problem to an \math{O(d^3 \varepsilon^{-2} \ln(\varepsilon^{-1})) \times d}
quantile regression problem. This smaller problem can be solved more quickly, and its solution is a $(1+\varepsilon)$-accurate
solution to the original quantile regression problem. The algorithm uses sparse Cauchy embeddings, which results in a
time cost of \math{O(\textsc{nnz}(A) \cdot \ln(n) + \text{poly}(d))} for computing the reduction, where \math{ \text{poly}(d) = \tilde{O}(d^{21/2} \ln^{5.5}(d))}%
\footnote{This runtime assumes that the approximate ellipsoidal
rounding of~\cite{malik209} is used in the final step of Algorithm~2
of~\cite{YMM2014}.}.  
The authors of~\cite{WW2018} show that, by substituting
an oblivious \math{\ell_1}-embedding introduced in that work for the sparse
Cauchy transform originally used in~\cite{YMM2014}, one can obtain a more favorable
dependence on \math{d}, namely \math{\text{poly}(d) = \tilde{O}(d^{13/2} \ln^{5/2} d)}%
\footnote{\cite{WW2018} presents two oblivious \math{\ell_1}-embeddings:
one with two non-zero entries per row, and one with a variable number
\math{B} per row; we use the latter with \math{B = \ln d} to obtain this
running time. The former gives \math{p = O(d^7 \ln d)}.}.

The second \math{\ell_1}-embedding introduced in Theorem~\ref{theorem:L1-subspace} can be used in place of the embedding from~\cite{WW2018} to further reduce the runtime dependence on \math{d}
to \math{\text{poly}(d) = \tilde{O}(d^{13/2} \ln^{3/2+\eta}(d)/\ln\ln(d))}, where
\math{\eta \in (0,1/2)} is arbitrary. This represents an $\tilde{O}(\ln^{1/2 - \eta}(d) \ln\ln d)$ improvement in the $d$-dependent portion of the runtime over the 
current state-of-the-art algorithm given in~\cite{WW2018}.

%% file: RunTime.tex
%
\begin{tikzpicture}

\begin{axis}[%
width=4.521in,
height=3.566in,
at={(0.758in,0.481in)},
scale only axis,
xmin=2,
xmax=30,
xtick={10,20,30},
xlabel style={font=\color{white!15!black}},
xlabel={Embedding Dimension $r/d$},
ymin=0,
ymax=3.4,
ytick={1,2,3},
yticklabels={{$\times1$},{$\times2$},{$\times3$}},
ylabel style={font=\color{white!15!black}},
ylabel={Speedup w.r.t. Gaussian},
axis background/.style={fill=white},
font=\huge,xlabel style={font=\color{black},scale=1.9},ylabel style={font=\color{black},scale=1.9},ticklabel style={font=\scshape,scale=1.7},
]
\addplot [color=red, line width=3.0pt, forget plot]
  table[row sep=crcr]{%
2	0.254693377682592\\
14	1.74588259362477\\
18	2.24090494406281\\
20	2.48727082801068\\
24	2.98174223644046\\
28	3.47529200849877\\
};
\addplot [color=green, line width=3.0pt, forget plot]
  table[row sep=crcr]{%
2	0.125262287347933\\
18	1.10307856398505\\
30	1.83271637415552\\
};
\addplot [color=blue, line width=3.0pt, forget plot]
  table[row sep=crcr]{%
2	0.251931234451359\\
4	0.495298388804034\\
6	0.736766851414934\\
8	0.975430494725213\\
10	1.21178082081796\\
12	1.4458802069452\\
14	1.67757243307166\\
16	1.90433069091529\\
18	2.1330600303493\\
20	2.35541605765288\\
22	2.57920009960828\\
26	3.01640839463379\\
28	3.23204470453444\\
30	3.44420779432512\\
};
\node[right, align=left]
at (axis cs:20.5,2.3) {This Paper};
\node[right, align=left]
at (axis cs:10,3) {$\text{sample}_r(\mat{RHT}(\matA))$};
\node[right, align=left]
at (axis cs:15,0.75) {$\text{sample}_r(\mat{RHT}^2(\matA))$};
\end{axis}
\end{tikzpicture}%

%% file: Distortion.tex
%
\begin{tikzpicture}

\begin{axis}[%
width=4.521in,
height=3.566in,
at={(0.758in,0.481in)},
scale only axis,
xmin=2,
xmax=30,
xtick={10, 20, 30},
xlabel style={font=\color{white!15!black}},
xlabel={Embedding Dimension $r/d$},
ymin=0,
ymax=1.5,
ytick={  0, 0.5,   1, 1.5},
ylabel style={font=\color{white!15!black}},
ylabel={Distortion ${\varepsilon/\varepsilon_{\textsc{g}}}$\raisebox{-13pt}{}},
axis background/.style={fill=white},
font=\huge,xlabel style={font=\color{black},scale=1.9},ylabel style={font=\color{black},scale=1.9},ticklabel style={font=\scshape,scale=1.7},
]
\addplot [color=red, line width=3.0pt, forget plot]
  table[row sep=crcr]{%
2	1.40121322998538\\
4	1.41882831467383\\
6	1.42817202635428\\
8	1.43132795493838\\
10	1.43208562634321\\
12	1.42948677951255\\
14	1.43130991450781\\
16	1.42969322700985\\
18	1.42650328093548\\
20	1.42783897579698\\
24	1.42575740742399\\
28	1.42346473025427\\
30	1.42045311356539\\
};
\addplot [color=green, line width=3.0pt, forget plot]
  table[row sep=crcr]{%
2	0.99601453016734\\
6	0.990305466824505\\
20	0.973971328996264\\
30	0.962407496893018\\
};
\addplot [color=blue, line width=3.0pt, forget plot]
  table[row sep=crcr]{%
2	0.372211664825663\\
4	0.386552994242503\\
6	0.389337262264061\\
8	0.388408266209105\\
10	0.384899626503127\\
12	0.379897863734246\\
14	0.372922695531408\\
16	0.365265099366766\\
18	0.356861764970006\\
22	0.337691723202667\\
24	0.327045951106708\\
26	0.315764279647265\\
30	0.29126600630179\\
};
\addplot [color=blue, line width=3.0pt, forget plot]
  table[row sep=crcr]{%
2	1.10197829555148\\
4	1.08333438585097\\
6	1.07306995542074\\
8	1.06638131522799\\
10	1.06079749593511\\
14	1.05211417184308\\
24	1.03488274462226\\
30	1.02649058143451\\
};
\addplot [color=black, line width=2.0pt, forget plot]
  table[row sep=crcr]{%
2	1\\
30	1\\
};
\node[right, align=left]
at (axis cs:7,1.14) {(\rn{1}) $\matA\mapsto\Pi\transp_{d\times r}(\Atilde\transp\Atilde)^{1/2}$};
\node[right, align=left]
at (axis cs:7,0.47) {(\rn{2}) This Paper: $\matA\mapsto(\Atilde\transp\Atilde)^{1/2}$};
\node[right, align=left]
at (axis cs:7,1.33) {$\text{sample}_r(\mat{RHT}(\matA))$};
\node[right, align=left]
at (axis cs:7,0.85) {$\text{sample}_r(\mat{RHT}^2(\matA))$};
\end{axis}
\end{tikzpicture}%

%% file: Related.tex
\section{Related Work}

\paragraph{Related work for \math{\ell_2}.}
There is a long history of oblivious embedding techniques starting
with the random Gaussian projection based on the famed Johnson-Lindenstrauss
lemma~\cite{JL1984}. Since algorithmic applications of metric embeddings
were first considered in \cite{LLR1995} there has been an explosion of
techniques, in particular with the arrival of the fast JLT for
solving approximate nearest neighbor problems in \cite{AC2006,AC2009}, which
was then applied to PCA, \math{\ell_2}-regression and
matrix multiplication in \cite{S2006}.

In parallel, row-sampling as a means
to construct subspace-embeddings to preserve \math{\ell_2} structure originated
in statistics~\cite{RV2007,T1990,T2012} and received much attention in
numerical linear
algebra~\cite{DKM2006a,DKM2006b,DKM2006c,DMM2006,LMP2013,malik145}
and in graph spectral sparsification~\cite{BSS2009}.
The sampling probabilities are often based on
the \math{\ell_2}-leverage scores,
and fast approximation of leverage scores based on fast JLTs were given in
\cite{malik186} and analysis based on matrix-Chernoff bounds
shows that
\math{O(d\varepsilon^{-2}\ln (d/\varepsilon))} rows suffices to get
\math{(1+\varepsilon)}-approximate algorithms,
in relative error,~\cite{AW2002,malik145,RV2007,T2012} .
The slow, but polynomial, deterministic algorithm in~\cite{BSS2009} shows
that \math{O(d/\varepsilon^2)} rows suffice.
Recently, relative error algorithms have been shown possible in input-sparsity
time through the use of sparse oblivious \math{\ell_2}-subspace
embeddings~\cite{CW2013,NN2013,MM2013}. All such embedding based approaches
require a number of rows or embedding dimension that is
\math{O(1/\varepsilon^2)}. We are the first to offer
\math{\ell_2}-embeddings of arbitrary precision \math{\varepsilon}
in a fixed dimension \math{d}. Our embedding can leverage
any other \math{\ell_2}-embedding as a black-box, and in particular can
achieve fast/input-sparsity runtimes
(up to \math{\text{poly}(d)} additive terms)
by using fast/input-sparsity embeddings as in
\cite{CW2013,NN2013,T2011}. Our embedding is simple, but nonlinear,
hence slight modifications are needed to apply it to
get fast runtimes for applications: regression, reconstruction,
leverage scores.

\paragraph{Related work for \math{\ell_p}, \math{p\not=2}.}
There is also substantial progress for values of \math{p\not=2},
\cite{malik188,malik209,DDHKM2009,MM2013,SW2011,WW2018}.
Of particular interest is \math{p=1}, which in machine learning applications
corresponds to robust regression.
For the \math{\ell_1} case, sampling based algorithms that use a
well-conditioned basis give embeddings with \math{O(d^{2.5})} rows, but such
algorithms are slow. In~\cite{T1990} it is shown that
\math{O(d\ln d)} rows suffice.
The state-of-the-art for 
fast algorithms using row sampling probabilities
based on \math{\ell_1}-leverage scores is achieved by using the 
oblivious \math{\ell_1}-embeddings in~\cite{WW2018}
which have input-sparsity runtimes but
produce embeddings with \math{\tilde O(d^{3.5})} rows.
In terms of number of rows, the
best algorithm uses Lewis weights for 
row sampling, \cite{CP2015}, which embeds into 
\math{O(d\ln d)} rows. This algorithm is non-oblivious and requires
\math{O(\ln\ln n)} passes through the data.
We apply our \math{\ell_2}-embedding as a black-box
to get
an oblivious \math{\ell_1} embedding using the methods in
\cite{WW2018}. Our analysis accomodates an arbitrary \math{\ell_2} embedding
while the analysis in
\cite{WW2018} is specific to the \textsf{CountSketch} and \textsf{OSNAP}
\math{\ell_2} embeddings which also preserve \math{\ell_1} dilation.
The best oblivious embedding dimension of \math{O(d\ln d)}
from \cite{WW2018} uses the
\textsf{OSNAP} approach and gives \math{O(d\ln d)} distortion. To
get \math{O((d\ln d)/\ln\ln d)} distortion with the
\textsf{OSNAP} approach requires embedding dimension
\math{O(d\ln^2 d)} whereas we only need \math{O(d\ln^{1+\eta} d)}
dimensions:
our runtimes are comparably fast and our embedding dimension is slightly
tighter for comparable distortion. We believe our techniques
can be used to give a more refined analysis of the
\textsf{OSNAP} approach to get similar asymptotic performance
to ours. We did not pursue this avenue since
the \math{\ell_2} part of our embedding has tighter dimension
(\math{d} for ours versus \math{O(d\ln d)} for \textsf{OSNAP}).

%% file: Experiments.tex
\section{Experimental Demonstration}

We demonstrate the theory
using the $\ell_1$-embedding (\texttt{|MG|})
in \r{eq:embedding-def-intro} applied to
fast $\ell_1$-regression. We compare with three
other embeddings:
\begin{itemize}[itemindent=-15pt]\itemsep=-1pt
  \vspace*{-6pt}
  \item Embeddings formed by sampling rows uniformly at random without replacement (\texttt{|Unif|})
  \item Embeddings formed using Lewis weights (\texttt{|Lewis|})
  \item Embeddings formed using the Wang-Woodruff $\ell_1$-embedding (\texttt{|WW|})
  \vspace*{-6pt}
\end{itemize}
The \texttt{|Lewis|} embeddings are formed following the procedure described in
Lemma~\ref{lemma:lewis-intro}\footnote{The leverage scores are exactly
  computed, rather than approximated using a JLT, because $ d \ll 300\ln(n)$.}.
The \texttt{|WW|} and \texttt{|MG|} embeddings are formed
according to the discussion preceeding Theorem~\ref{theorem:l1-sampling}:
$\left[\matA \ \bb\right]$ is embedded into an $r \times d$ matrix that is used to construct a well-conditioned basis
$\matU$ for $\left[\matA \ \bb\right]$, then $r$ constraints are sampled proportionally to the
$\ell_1$-norms of $\matU$ to obtain the reduced regression problem.

We construct a simple regression problem
that is challenging for na\"ive sampling
schemes by adding many non-informative and noisy constraints.
Specifically, $\matA \in \R^{n \times d}$ and $\bb \in \R^n$ are
\mldc{%
  \matA=
  \left[
  \begin{matrix}
    \ee_1\ee_1^T +  (\matI - \ee_1 \ee_1^T) \matG_1 (\matI - \frac{1}{d}\v1\v1^T) \\
    \vdots \\
    \ee_d\ee_d^T +  (\matI - \ee_d \ee_d^T) \matG_d (\matI - \frac{1}{d}\v1\v1^T) \\
    \matG (\matI - \frac{1}{d}\v1\v1^T) 
  \end{matrix}
  \right]
  \quad \text{ and } \quad 
  \bb = \left[
  \begin{matrix}
  \alpha \textrm{e}_1 + \varepsilon (\mathbf{I} - \textrm{e}_1 \textrm{e}_1^T) \v{g}_1 \\
  \vdots \\
  \alpha \textrm{e}_d + \varepsilon (\mathbf{I} - \textrm{e}_d \textrm{e}_d^T) \v{g}_d \\
  \varepsilon \v{g}
  \end{matrix}
  \right],
}
where $\varepsilon, \alpha\ge0$ and 
the matrices $\matG_i \in \R^{d \times d}$ contain independent
Gaussians, as does $\matG \in \R^{(n-d^2) \times
  d}$. Similarly, $\v{g}_i \in \R^{d}$ and $\v{g} \in \R^{n-d^2}$ are standard
Gaussian vectors.
Uniform sampling would pick many non-informative constraints for this
problem, and hence perform poorly.
%
%
%
%
We used $d = 70$, $n \approx d^3$,
$\varepsilon = 1/\sqrt{n}$, and $\alpha = 20$, and the embedding
dimension $r$ varied in the range $[3d,30d]$.
The results average errors over several independent trials are reported in
the figures below.
\\
\centerline{\tabcolsep0pt
  \begin{tabular}{p{0.5\textwidth}p{0.5\textwidth}}
\scalebox{0.425}{\includegraphics{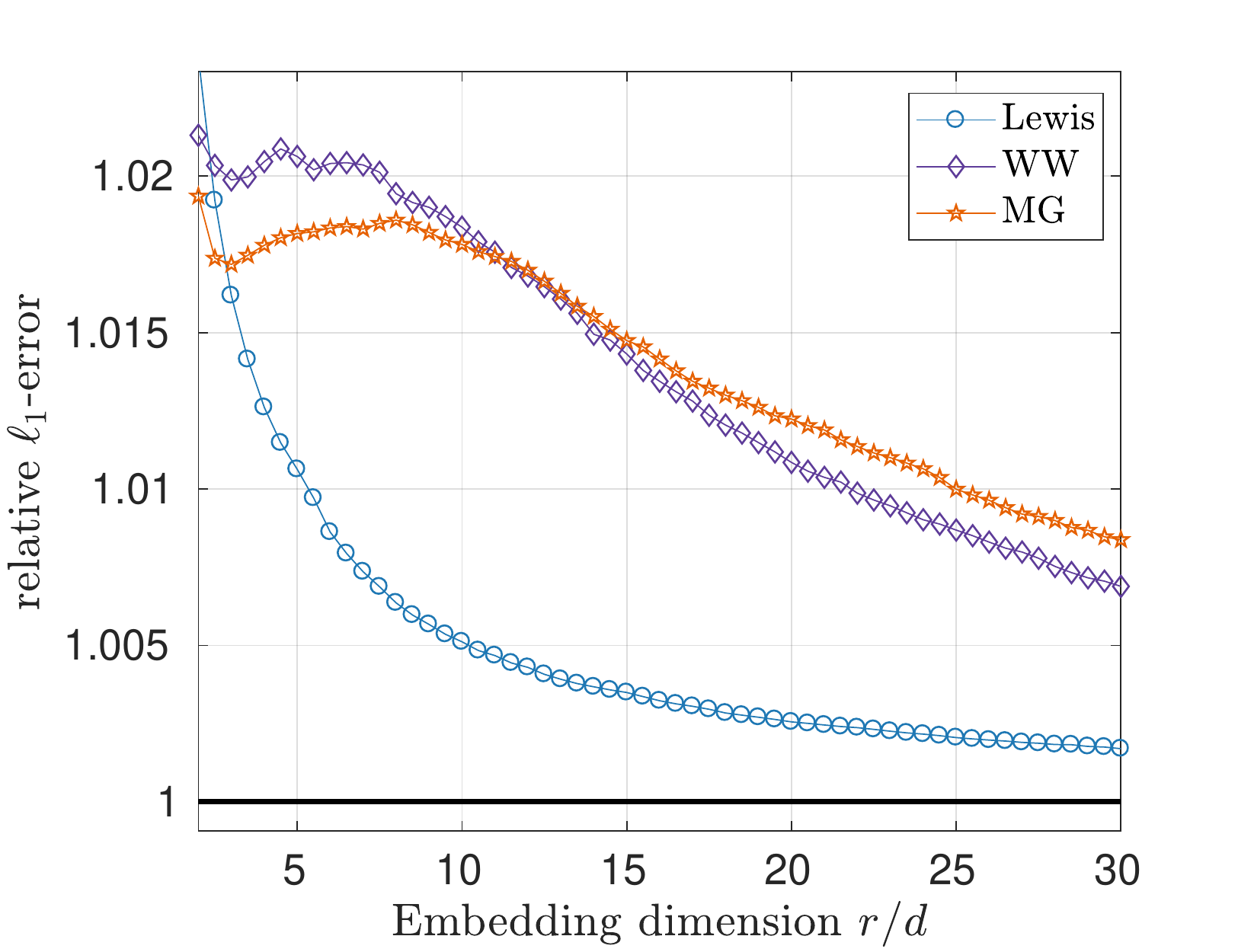}}
&
\scalebox{0.425}{\includegraphics{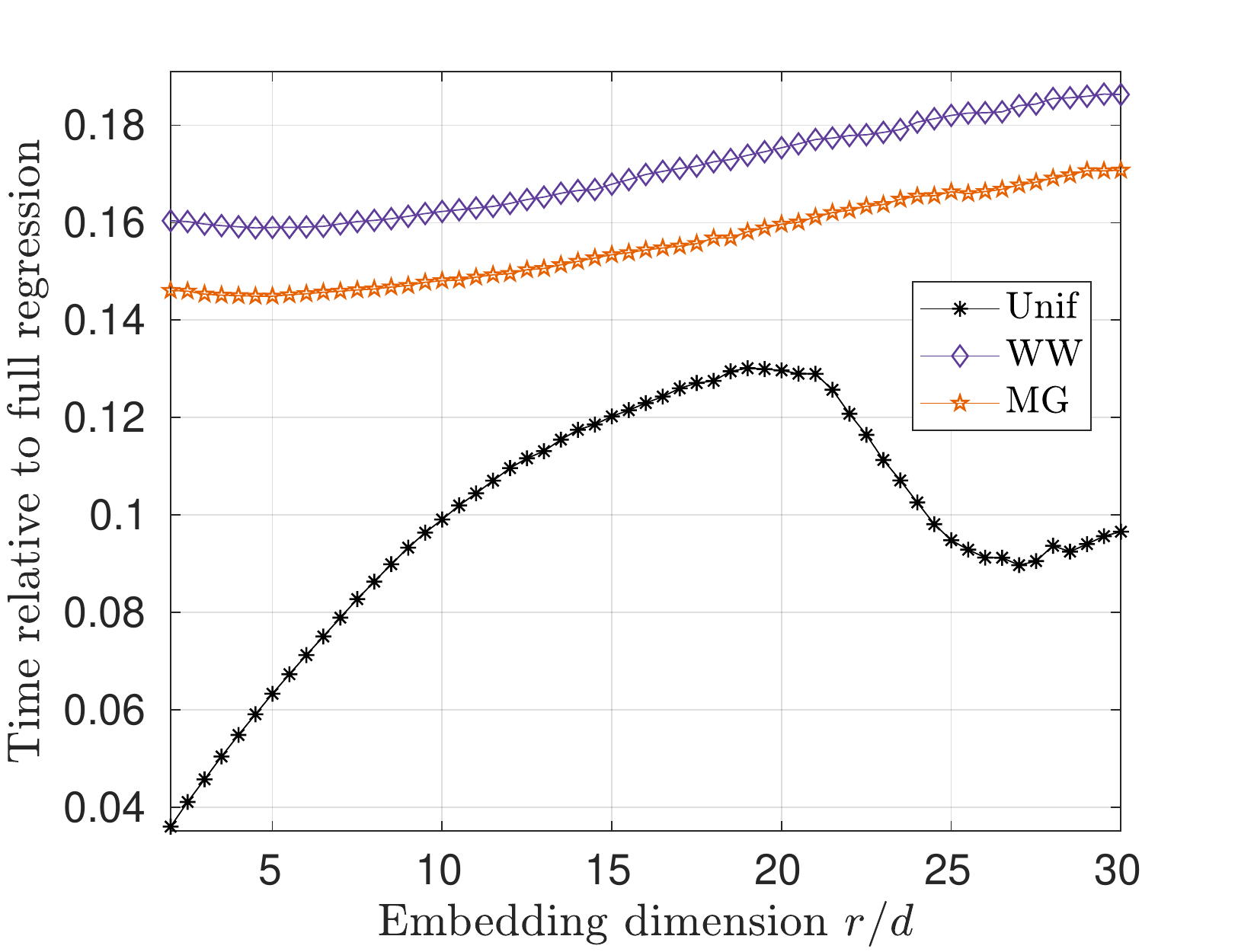}}
\\
\multicolumn{1}{c}{(a) Relative \math{\ell_1}-error embeddings} &
\multicolumn{1}{c}{(b) Runtime} 
\end{tabular}
}
\\[2pt]
In (a) we show
\math{\norm{\matA\tilde\xx-\bb}_1/\norm{\matA\tilde\xx_{\text{opt}}-\bb}_1}.
Uniform sampling, as expected,
is off the scale (even for \math{r=30d}, uniform
is \math{2\frac{1}{2}\times} worse than optimal).
Performance of \texttt{|WW|}  and \texttt{|MG|} are comparable, while
Lewis-sampling is slightly better, all achieving close to 1 in
relative error.
In (b) we show runtime.
Lewis-sampling, which is off the scale,
\math{7\frac{1}{2}\times}
\emph{slower} than the full regression. Even in this non-asymptotic
regime, there are significant gains from \texttt{|WW|} and \texttt{|MG|}.
In the asymptotic regime, we expect Lewis sampling to become more competitive
and \texttt{|MG|} will be slightly more efficient than \texttt{|WW|} due to its
tighter embedding dimension.

%% file: L2.tex
\section{Proofs: \math{\ell_2}-Embedding and Applications}
Our embedding uses, as a black box, any fast \math{\ell_2}-subspace embedding.
The most relevant such embeddings are: for dense matrices, the fast subsampled randomized Hadamard transform,
(for dense matrices), and for sparse matrices, the \textsf{CountSketch} and
\textsf{OSNAP} embeddings.

\subsection{Fast-Hadamard Embedding}
Tropp~\cite{T2011} gave a tight analysis of the Hadamard transform.
We give a detailed statement of this result in which we
explicitly state some constants that are otherwise hard to identify directly from the
theorems as stated in~\cite{T2011}.
The next lemma is a straightforward application of Lemmas~3.3 and~3.4 in~\cite{T2011};
the runtime to compute \math{\Pi_{\textsc{h}}\transp\matA} is established in~\cite{AL2013}.
\begin{lemma}[{\cite[Lemmas 3.3 and 3.4]{T2011}}]\label{lemma:SRHT-tropp}
  Let
  \math{\matU\in\R^{n\times d}}
  be an orthogonal matrix and \math{\Pi_\textsc{h}\in\R^{r\times n}} be a subsampled
  randomized Hadamard transform matrix. For \math{0<\varepsilon\le\frac12},
  suppose the embedding dimension \math{r}
  satisfies:
  \mld{
    r\ge \frac{12}{5\varepsilon^2}\left(\sqrt{d}+\sqrt{8\ln(3t n)}\right)^2\ln d
    \ \in\ 
    O\left(\frac{1}{\varepsilon^2}(d+\ln(n))\ln d\right).
  }
  Then, with probability at least \math{1-1/t},
  \math{\Pi_\textsc{h}} is an \math{\epsilon}-JLT for \math{\matU},
  \mld{
    \norm{\matI-\matU\transp\Pi_\textsc{h}\Pi_\textsc{h}\transp\matU}_2\le\varepsilon.
  }
  Further, the product \math{\Pi_\textsc{h}\transp \matA} can be computed in
  time \math{O(nd\ln r)} for any matrix \math{\matA\in\R^{n\times d}}.
\end{lemma}

\subsection{Proof of Theorem~\ref{theorem:L2-subspace}: \math{\ell_2}-Subspace
Embedding}
Suppose that \math{\Pi\in\R^{r\times d}} is an \math{\varepsilon}-JLT for
\math{\matU}, where \math{\matA=\matU\Sigma\matV\transp}. Such a
\math{\Pi} can be the randomized Hadamard transform from
Lemma~\ref{lemma:SRHT-tropp} or a \textsf{CountSketch} or \textsf{OSNAP}
embedding matrix of appropriate dimensions. Let
\math{\Atilde=(\matA\transp\Pi\Pi\transp\matA)^{1/2}}. Then,
\eqar{
  |\norm{\matA\xx}_2^2- \norm{\Atilde\xx}_2^2|
  &=&
  |\xx\transp\matA\transp\matA\xx-
  \xx(\matA\transp\Pi\Pi\transp\matA)^{1/2}(\matA\transp\Pi\Pi\transp\matA)^{1/2}
  \xx|\\
  &=&
  |\xx\transp\matA\transp\matA\xx-
  \xx\matA\transp\Pi\Pi\transp\matA\xx|\\
  &=&
  |\xx\transp\matV\Sigma(\matU\transp\matU-\matU\transp\Pi\Pi\transp\matU)\Sigma
  \matV\transp\xx|\\
  &=&|\xx\transp\matV\Sigma(\matI-\matU\transp\Pi\Pi\transp\matU)\Sigma
  \matV\transp\xx|\\
  &\le&
  \norm{\matI-\matU\transp\Pi\Pi\transp\matU}_2\norm{\Sigma
    \matV\transp\xx}_2^2\\
  &\le&\varepsilon\norm{\matA\xx}_2^2.
}
The last step follows because \math{\Pi} is an \math{\varepsilon}-JLT for
\math{\matU} and \math{\norm{\Sigma
    \matV\transp\xx}_2=\norm{\matU\Sigma
    \matV\transp\xx}_2} (because \math{\matU} has orthonormal columns).
\qedsymb

\subsection{Proof of Theorem~\ref{theorem:L2-regression-intro}:
\math{\ell_2}-Regression}

We prove a more general claim, corresponding to the multiple regression problem.
Suppose one wishes to minimize
\math{\norm{\matW\matZ}_F^2+\Phi(\matZ)} over \math{\matZ\in\cl C},
where \math{\matW\in\R^{n\times q}} is an arbitrary but known matrix, 
\math{\matZ\in\R^{q\times p}} is the optimization variable---a 
\emph{matrix}---and \math{\Phi(\cdot)} is a nonnegative regularizer. Let \math{\matZ_*} be
an optimal \math{\matZ},
\mld{
  \matZ_* \in \mathop{\arg\min}_{\matZ\in\cl C}\norm{\matW\matZ}_F^2+\Phi(\matZ).
  \label{eq:xstar-app}
}
Let
\math{\tilde\matW=(\matW\transp\Pi\Pi\transp\matW)^{1/2}} be a nonlinear sketch of 
\math{\matW}, and let \math{\tilde\matZ} be optimal with respect to
\math{\tilde\matW},
\mld{
  \tilde\matZ \in \mathop{\arg\min}_{\matZ\in\cl C}\norm{\tilde\matW\matZ}_F^2+\Phi(\matZ).\label{eq:xtilde-app}
}
First, observe that for all \math{\matZ\in\R^{q\times p}},
\mld{
  (1-\varepsilon)\norm{\matW\matZ}_F^2
  \le
  \norm{\tilde\matW\matZ}_F^2
  \le
  (1+\varepsilon)\norm{\matW\matZ}_F^2,\label{eq:isometry-general-F}
}
since \math{\norm{\tilde\matW\matZ}_F^2=\sum_{j\in[p]}
  \norm{\tilde\matW\matZ^{(j)}}_2^2}, and the inequalities hold for each
column \math{\matZ^{(j)}}. We prove \math{\tilde\matZ} is a relative
error approximation to \math{\matZ_*}.
\begin{theorem}[\math{\ell_2}-Regression]\label{theorem:L2-regression-general}
  Let \math{\matZ_*} and \math{\tilde\matZ} be as defined in
  \r{eq:xstar-app} and \r{eq:xtilde-app} respectively. Then,
    \mld{\norm{\matW\tilde\matZ}_F^2+\Phi(\tilde\matZ)\le \frac{1+\varepsilon}{1-\varepsilon}\cdot\left(\norm{\matW\matZ_*}_F^2+\Phi(\matZ_*)\right).}
\end{theorem}
\begin{proof}
  \eqar{
    \norm{\matW\tilde\matZ}_F^2
    +\Phi(\tilde\matZ)
    &\buildrel (a) \over \le&
    \frac{1}{1-\varepsilon}\cdot\left(\norm{\tilde\matW\tilde\matZ}_F^2
    +\Phi(\tilde\matZ)\right)\\
    &\buildrel (b) \over \le&
    \frac{1}{1-\varepsilon}\cdot\left(\norm{\tilde\matW\matZ_*}_F^2
    +\Phi(\matZ_*)\right)\\
    &\buildrel (a) \over \le&
    \frac{1+\varepsilon}{1-\varepsilon}\cdot\left(\norm{\matW\matZ_*}^2
    +\Phi(\matZ_*)\right)
  }
where (a) uses \r{eq:isometry-general-F} and (b)
  is because \math{\tilde\matZ} is optimal for
  \math{\tilde\matW} and \math{\Phi\ge 0}.  
\end{proof}
We now prove the multiple regression version of
Theorem~\ref{theorem:L2-regression-intro} by showing it is a special case of
Theorem~\ref{theorem:L2-regression-general}.
Given \math{\matA\in\R^{n\times d}} and \math{\matB\in\R^{n\times p}}, we show
how to approximate \math{\matX_*} which minimizes
\mld{\norm{\matA \matX-\matB}_F^2+\Phi(\matX).\label{eq:multiple-F}}
    (Theorem~\ref{theorem:L2-regression-intro}
  is a special case where \math{p=1} and \math{\matB=\bb}.)
  Let
\math{\matW=[\matA,-\matB]} and \math{\cl C=\cl D\times\{\matI\}}. Then,
\math{\matZ\in \cl C} is of the form
\math{\matZ=
  \left[\begin{smallmatrix}\matX\\\matI_{p\times p}\end{smallmatrix}\right]},
and \math{\matW\matZ=\matA\matX-\matB},
where \math{\matX\in\cl D}.
Let \math{\matZ_*} minimize \math{\matW\matZ+\Psi(\matZ)},
where \math{\Psi(\matZ)=\Phi(\matZ_{1:d,1:p})}, and write
\math{\matZ_*=
  [\begin{smallmatrix}\matX_*\\\matI_{p\times p}\end{smallmatrix}]},
where \math{\matX_*\in\cl D} and
\math{\matX_*} minimizes \r{eq:multiple-F}.
Let
\mld{\tilde\matW=(\matW\transp\Pi\Pi\transp\matW)^{1/2}\in\R^{(d+p)\times
    (d+p)}}
be our embedding of \math{\matW} which
satisfies~\r{eq:isometry-general-F}, and define
\math{\Atilde\in\R^{(d+p)\times d}} and \math{\tilde\matB\in\R^{(d+p)\times p}} by
\math{\tilde\matW=[\Atilde,-\tilde\matB]}. Also,
\math{\tilde\matX} minimizes
\math{\norm{\Atilde \matX-\tilde\matB}_F^2+\Phi(\matX)} if and only if
\math{\tilde\matZ=[\begin{smallmatrix}\tilde\matX\\\matI_{p\times p}\end{smallmatrix}]} minimizes \math{\norm{\tilde\matW\matZ}_F^2+\Psi(\matZ)}. Using
Theorem~\ref{theorem:L2-regression-general} and \math{\Psi(\matZ)=\Phi(\matX)},
\mld{\renewcommand{\arraystretch}{1.75}
  \begin{array}{rcl}
  \norm{\matA\tilde\matX-\matB}_F^2+\Phi(\tilde\matX)
  &=&\norm{\matW\tilde\matZ}_F^2+\Phi(\tilde\matZ)\\
  &\le&\displaystyle \frac{1+\varepsilon}{1-\varepsilon}\cdot\left(\norm{\matW\matZ_*}_F^2+\Phi(\matZ_*)\right)\\
  &=&\displaystyle \frac{1+\varepsilon}{1-\varepsilon}\cdot\left(\norm{\matA \matX_*-\matB}_F^2
  +\Phi(\matX_*)\right).
  \end{array}
}
\qedsymb

\subsection{Proof of Theorem~\ref{theorem:L2-PCA-intro}:
  Relative Error  Low Rank \math{\ell_2}-Reconstruction (PCA)}

Since \math{\Pi} is an \math{\varepsilon}-JLT for \math{\matU},
by Theorem~\ref{theorem:L2-subspace},
\mld{
  (1-\varepsilon)\cdot\xx\transp\matA\transp\matA\xx
  \le
  \xx\transp\Atilde\transp\Atilde\xx
  \le 
  (1+\varepsilon)\cdot\xx\transp\matA\transp\matA\xx.
  \label{eq:isometry-AA-app}
}
For \math{k\in[d]}, let \math{\sigma_k} and
\math{\tilde\sigma_k} be the \math{k}th
largest singular value of \math{\matA} and \math{\Atilde} respectively.
We use the following lemma which follows from the
Courant-Fisher characterization of eigenvalues to bound
\math{\tilde\sigma_k^2}.
\begin{lemma}\label{lemma:sing-courant-fisher}
  For \math{k\in[d]},
  \mld{
    (1-\varepsilon)\cdot \sigma_k^2
    \le
    \tilde\sigma_k^2
    \le
    (1+\varepsilon)\cdot \sigma_k^2.
    \label{eq:sig-courant-fisher}
  }
\end{lemma}
\begin{proof}
  Let \math{\lambda_k} and
  \math{\tilde\lambda_k} be the \math{k}th largest eigenvalue of
  \math{\matA\transp\matA} and \math{\Atilde\transp\Atilde}
  respectively. By
  Courant-Fisher,
  \eqar{
    \tilde\sigma_k^2=\tilde\lambda_{k}(\Atilde\transp\Atilde)
    &=&
    \max_{\rank(\matW)=k}\min_{
        \norm{\xx}_2=1\atop
        \xx\in\matW
    }
    \xx\transp\Atilde\transp\Atilde\xx
    \\
    &\le&
    \max_{\rank(\matW)=k}\min_{
      \norm{\xx}_2=1\atop
      \xx\in\matW
    }
    (1+\varepsilon)\cdot  \xx\transp\matA\transp\matA\xx
    \\
    &=&
    (1+\varepsilon)\cdot\lambda_{k}=(1+\varepsilon)\cdot\sigma_{k}^2.
  }
  The lower bound follows in a similar manner. We omit the details.
\end{proof}
Now, to prove the theorem, let
\math{\tilde\matV_k} be the top-\math{k} right singular vectors of
\math{\Atilde}. Define
\math{\tilde\matP_k=\matI-\tilde\matV_k\tilde\matV_k\transp}, the
projector orthogonal to the space spanned by \math{\tilde\matV_k}. Then,
\math{\matA-\hat\matA_k=\matA\tilde\matP_k} and we have:
\eqar{
  \norm{\matA-\hat\matA_k}_2^2
  =
  \norm{\matA\tilde\matP_k}_2^2
  &=&
  \max_{\norm{\zz}=1}\zz\transp\tilde\matP_k\matA\transp\matA\tilde\matP_k\zz\\
  &\buildrel (a) \over \le&
  \max_{\norm{\zz}=1}
  \frac{\zz\transp\tilde\matP_k\Atilde\transp\Atilde\tilde\matP_k\zz}{1-\varepsilon}\\
  &=&
  \frac{1}{1-\varepsilon}\cdot\tilde\sigma_{k+1}^2
  \\
  &\buildrel (b) \over \le&
  \frac{1+\varepsilon}{1-\varepsilon}\cdot\sigma_{k+1}^2\\
  &=&
  \frac{1+\varepsilon}{1-\varepsilon}\cdot
  \norm{\matA-\matA_k}_2^2.
}
In (a) we used \r{eq:isometry-AA-app} with \math{\xx=\tilde\matP_k\zz}
and in (b) we used Lemma~\ref{lemma:sing-courant-fisher}.
\qedsymb

\subsection{Proof of Theorem~\ref{theorem:L2-leverage-intro}:
  Fast Approximation of \math{\ell_2}-Leverage Scores}

Let \math{\matA=\matU\Sigma\matV\transp}.
Since \math{\Pi} is an \math{\varepsilon}-JLT for \math{\matU},
\mld{\norm{\matI-\matU\transp\Pi\Pi\transp\matU}_2\le\varepsilon.
\label{eq:e-JLT-app}}
We will need the following lemma.
\begin{lemma}\label{lemma:e-JLT-inverse}
  \math{\displaystyle
    \norm{\matI-(\matU\transp\Pi\Pi\transp\matU)^{-1}}_2\le\frac{\varepsilon}{1-\varepsilon}.}
\end{lemma}
\begin{proof}
  Let \math{\lambda_1\ldots,\lambda_d} be the eigenvalues of
  \math{\matU\transp\Pi\Pi\transp\matU}.
  By \r{eq:e-JLT-app},
  \math{1-\varepsilon\le\lambda_i\le 1+\varepsilon}, which implies
  \mld{
    \frac{1}{1+\varepsilon}\le
    \lambda_i((\matU\transp\Pi\Pi\transp\matU)^{-1})
    \le
    \frac{1}{1-\varepsilon},
  }
  which in-turn implies that  
  \mld{
    \frac{-\varepsilon}{1-\varepsilon}\le
    \lambda_i(\matI-(\matU\transp\Pi\Pi\transp\matU)^{-1})
    \le
    \frac{\varepsilon}{1+\varepsilon}.
  }
  Since \math{\norm{\matI-(\matU\transp\Pi\Pi\transp\matU)^{-1}}_2=\max_i
    |\lambda_i(\matI-(\matU\transp\Pi\Pi\transp\matU)^{-1})|},
  the lemma follows.
\end{proof}
We now prove Theorem~\ref{theorem:L2-leverage-intro}
using
\math{\Atilde=(\matA\transp\Pi\Pi\transp\matA)^{1/2}
  =
  (\matV\Sigma\matU\transp\Pi\Pi\transp\matU\Sigma\matV\transp)^{1/2}}:
\eqar{
  \tilde\tau_i
  &=&
  \norm{\ee_i\transp\matA\Atilde^{-1}}_2^2\\
  &=&
  \ee_i\transp\matA\Atilde^{-2}\matA\transp\ee_i\\
  &=&
  \ee_i\transp\matU\Sigma\matV\transp
  (\matV\Sigma\matU\transp\Pi\Pi\transp\matU\Sigma\matV\transp)^{-1}
  \matV\Sigma\matU\transp\ee_i\\
  &=&
  \ee_i\transp\matU\Sigma\matV\transp\matV\Sigma^{-1}
  (\matU\transp\Pi\Pi\transp\matU)^{-1}\Sigma^{-1}\matV\transp
  \matV\Sigma\matU\transp\ee_i\\
  &=&
  \ee_i\transp\matU
  (\matU\transp\Pi\Pi\transp\matU)^{-1}\matU\transp\ee_i
  \\
  &=&
  \ee_i\transp\matU\matU\transp\ee_i-
  \ee_i\transp\matU(\matI-(\matU\transp\Pi\Pi\transp\matU)^{-1})\matU\transp\ee_i
}
Since
\math{
  \tau_i=\norm{\ee_i\transp\matA\matA^{\dagger}}_2^2=
  \ee_i\transp\matU\matU\transp\ee_i=\norm{\ee_i\transp\matU}_2^2},
we have that
\eqar{
  |\tilde\tau_i-\tau_i|
  &=&
  |\ee_i\transp\matU(\matI-(\matU\transp\Pi\Pi\transp\matU)^{-1})
  \matU\transp\ee_i|\\
  &\le&
  \norm{\ee_i\transp\matU}_2^2\cdot
  \norm{\matI-(\matU\transp\Pi\Pi\transp\matU)^{-1}}_2\\
  &=&
  \tau_i\cdot
  \norm{\matI-(\matU\transp\Pi\Pi\transp\matU)^{-1}}_2.
}
To conclude the proof, use Lemma~\ref{lemma:e-JLT-inverse}.
\qedsymb

%% file: L1.tex
\section{Proofs: \math{\ell_1}-Embedding and Applications.}

Before we prove Theorem~\ref{theorem:L1-subspace}, we state
several probability bounds which we need.

\subsection{Preliminary Tools: Probability Bounds}

\begin{lemma}[Conditioning]\label{lemma:conditional}
  For any two events \math{A} and \math{B},
  \math{\Prob[A]\ge 1-\Prob[\overline{A}\mid B]-\Prob[\overline{B}].}
\end{lemma}
\begin{proof}
  By the law of total probability,
  \eqan{
    \Prob[A]
    &=&
    \Prob[A\mid B]\Prob[B]+\Prob[A\mid \overline{B}]\Prob[\overline{B}]\\
    &\ge&
    \Prob[A\mid B]\Prob[B]\\
    &=&(1-\Prob[\overline{A}\mid B])(1-\Prob[\overline{B}])\\
    &=&1-\Prob[\overline{A}\mid B]-\Prob[\overline{B}]+\Prob[\overline{A}\mid B]
    \Prob[\overline{B}]\\
    &\ge&1-\Prob[\overline{A}\mid B]-\Prob[\overline{B}].
  }
\end{proof}

\begin{lemma}[Hoeffding]\label{lemma:hoeffding}
  Let \math{\rvX_i} be independent random variables with
  \math{a_i\le\rvX_i\le b_i}. Let
  \math{\rvX=\sum_{i\in[n]}\rvX_i} and \math{B^2=\sum_{i\in[n]}(b_i-a_i)^2}.
  Then, for any \math{t>0},
  \mld{\Prob\left[\rvX-\Exp[\rvX]\ge t\right]
    \le
    \exp\left(-2t^2/B^2\right).
    }
\end{lemma}

\begin{lemma}[Chernoff]\label{lemma:chernoff}
  Let \math{\rvX_i} be independent random variables in
  \math{[0,1]} and let \math{\rvX=\sum_{i\in[n]}\rvX_i}. Then, for any \math{t>0},
  \mld{
    \begin{array}{rcl}
    \Prob\left[\rvX>(1+t)\Exp[\rvX]\right]
    &\le&
    \exp\left(-{\textstyle\frac13}t^2\Exp[\rvX]\right);\\[5pt]
    \Prob\left[\rvX<(1-t)\Exp[\rvX]\right]
    &\le&
    \exp\left(-{\textstyle\frac12}t^2\Exp[\rvX]\right).
    \end{array}
  }
\end{lemma}

\begin{lemma}[Maurer, \cite{M2003}]\label{lemma:maurer}
  Let \math{\rvX_i} be independent positive random variables with
  \math{\Exp[\rvX_i^2]\le \infty}.  Let
  \math{\rvX=\sum_{i\in[n]}\rvX_i} and \math{S^2=\sum_{i\in[n]}\Exp[\rvX_i^2]}.
  For any \math{t>0},
  \mld{\Prob\left[\rvX\le\Exp[\rvX]- t\right]
    \le
    \exp\left(-t^2/2S^2\right).
    }
\end{lemma}

A similar result to the one following was established in
\cite{WW2018}, which is an application of \cite{DD1996}.
\begin{lemma}[Negative Dependence]\label{lemma:negative}
  Let \math{\yy=[y_1,\ldots,y_n]} and
  let \math{S_1,\ldots,S_r} be a random partition of \math{[n]} into
  \math{r} disjoint sets. So each \math{i\in [n]} is
  independently placed
  into one of the \math{S_j} with each \math{S_j} being equally likely
  (having probability \math{1/r}). Let \math{\gamma_j=\sum_{i\in S_j}|y_i|}.
  If \math{\norm{\yy}_2<\norm{\yy}_1/\beta},
  \begin{enumerate}
  \item For any fixed \math{j}, \math{\gamma_j} cannot be too small:
    \mld{\Prob\left[\gamma_j\le \frac{\norm{\yy}_1}{2r}\right]
      \le \exp\left(-\frac{\beta^2}{8r}\right).}
  \item For any index-set \math{I\subseteq [r]},
    the probability that all \math{\gamma_j} are small for \math{j\in I} is
    bounded by the product,
    \mld{\Prob\left[\bigcap\limits_{j\in I}\left\{\gamma_j\le \frac{\norm{\yy}_1}{2r}\right\}\right]
      \le\prod\limits_{j\in I}\Prob\left[\gamma_j\le \frac{\norm{\yy}_1}{2r}\right]\le
      \exp\left(-\frac{|I|\beta^2}{8r}\right).}
  \item Let \math{\cl{E}} be the event that at least
    \math{r/k} of the \math{\gamma_j} are at most \math{\norm{\yy}_1/2r}, where
    \math{k>1}. Then,
    \mld{\Prob[\cl{E}]\le\exp\left(-\frac{\beta^2}{8k}+\frac{r}{k}(1+\ln k)\right).}
  \end{enumerate}
\end{lemma}
  \begin{proof}
By homogeneity, we may assume that \math{\norm{\yy}_1=1}.
  For \math{i\in[n]}, let \math{z_i} be the indicator of whether
  \math{i\in S_j}. Therefore,
  \math{\gamma_j=\sum_{i\in[n]}z_i|y_i|,} where \math{z_i} are independent
  Bernoulli indicator random variables with \math{\Prob[z_i=1]=1/r}.
  Let \math{\rvX_i=z_i|y_i|}, then
  \math{\gamma_j=\sum_{i\in[n]}\rvX_i}.
  Since \math{\Exp[z_i]=1/r} and \math{\Exp[z_i^2]=1/r},
  \math{\sum_{i\in[n]}\Exp[\rvX_i]=\sum_{i\in[n]}|y_i|/r=1/r} and
  \math{\sum_{i\in[n]}\Exp[\rvX_i^2]=\sum_{i\in[n]}|y_i|^2/r=\norm{\yy}_2^2/r\le
    1/\beta^2 r}.
  Setting \math{t=1/2r} in Maurer's bound, Lemma~\ref{lemma:maurer},
  \mld{
    \Prob[\gamma_j\le 1/2r]\le\exp(-\beta^2/8r),
  }
  which proves part 1 of the lemma.
  Part 2 follows from Proposition 4 in \cite{DD1996}
  because the \math{\gamma_j} are negatively
  associated, hence \math{1-\gamma_j} are also negatively associated.
  To prove Part 3, observe that the event \math{\cl E} implies that
  some subset of the \math{\gamma_j}, of size \math{r/k}, are all at most
  \math{1/2r}. There are \math{\choose{r}{r/k}\le (ek)^{r/k}} such subsets, so
  by the union bound and part 2 of the lemma,
  \mld{
    \Prob[\cl E]\le \choose{r}{r/k}\exp\left(-\frac{r\beta^2}{8rk}\right)
    \le
    (ek)^{r/k}\exp\left(-\frac{\beta^2}{8k}\right)
    =
    \exp\left(-\frac{\beta^2}{8k}+\frac{r}{k}(1+\ln k)\right).
    }
\end{proof}

  \begin{lemma}[Lower Tail of Independent Cauchys]\label{lemma:lower-cauchy}
    Let \math{\rvC_1,\ldots,\rvC_m} be \math{m} inpdependent
    Cauchy random variables. Then, with probability at least
    \math{1-e^{-m/8}}, \math{\sum_{i\in[m]}|\rvC_i|\ge m/4}.
  \end{lemma}
  \begin{proof}
    Let \math{z_i=1} if \math{|\rvC_i|\ge 1} and \math{z_i=0} otherwise.
    Then, \math{\sum_{i\in[m]}|\rvC_i|\ge\sum_{i\in[m]}z_i}, and the latter is a
    sum of independent Bernoullis with
    \math{\Prob[z_i=1]=\frac{2}{\pi}\int_{1}^\infty dx\frac{1}{1+x^2}=\frac12}.
    Using the Hoeffding bound,
    \math{\Prob[\sum_{i\in[m]}z_i\le m/2-t]\le \exp(-2t^2/m)}. The lemma
    follows by choosing \math{t=m/4}.
  \end{proof}
  To counteract a union bound over
  \math{O(\text{poly}(d)^d)} cases, we only need failure-probability
  \math{e^{-\Omega(d\ln d)}}.
  In the previous lemma, if \math{m\ge (c d\ln d)\cdot f(d)} for
  an increasing function \math{f(d)}, then we can trade-off some
  failure-probability
  for a better lower bound, while still achieving
  failure-probability \math{e^{-\Omega(d\ln d)}}. To exploit this tradeoff, we
  reformulate Lemma 2.12 in \cite{WW2018} as
  Lemma~\ref{lemma:lower-cauchy-2} below.
  \begin{lemma}[Lemma 2.12 in \cite{WW2018}]\label{lemma:lower-cauchy-2}
    Suppose \math{\rvC_1,\ldots,\rvC_m} are \math{m} inpdependent
    Cauchy random variables, with
    \math{m\ge (14c d\ln (td))\cdot e^{\frac43 f(td)}},
    for \math{c\ge 1} and \math{d\ge 2}.
    Then, with probability at least
    \math{1-2e^{-cd\ln (td)}},
    \mld{\sum_{i\in[m]}|\rvC_i|\ge
      {\textstyle\frac14 m(1+\frac34\floor{\frac43 f(td)})}
      \in\Omega(m\cdot f(td)).
    }
  \end{lemma}
  \begin{proof}
    Let \math{z_i=1} if \math{|\rvC_i|\ge e^\alpha} and \math{z_i=0} otherwise,
    where \math{\alpha\in\{0,1,2,\ldots\}}. Let \math{N_\alpha} be the number
    of \math{|C_i|} which are at least \math{e^\alpha},
    \math{N_\alpha=|\{C_i|C_i\ge e^\alpha\}|}.
    we have that, for all \math{\ell\in\{0,1,2,\ldots\}},
    \mld{
      \sum_{i\in[n]}|C_i|\ge N_0+\sum_{\alpha=1}^\ell N_\alpha(e^{\alpha}-e^{\alpha-1})
      = N_0+0.63\sum_{\alpha=1}^\ell N_\alpha e^{\alpha}.
      \label{eq:lower-cauchy-1}
    }
    To see this, observe that
    \eqan{
      \sum_{i\in[n]}|C_i|=\sum_{0\le |C_i|< 1}|C_i|+\sum_{\alpha\ge 0}\quad\sum_{e^{\alpha}\le
        |C_i|< e^{\alpha+1}}|C_i|.
      &=&
\sum_{\alpha=0}^{\ell}\quad\sum_{e^{\alpha}\le
  |C_i|< e^{\alpha+1}}|C_i|+\sum_{e^{\ell+1}\le
  |C_i|}|C_i|
      \\
      &\ge&
      \sum_{\alpha=0}^{\ell}(N_{\alpha}-N_{\alpha+1})e^\alpha+N_{\ell+1}e^{\ell+1}.
      \\
      &=&
      N_0+\sum_{\alpha=1}^\ell N_\alpha(e^{\alpha}-e^{\alpha-1}).
    }   
    By Lemma~\ref{lemma:lower-cauchy},
    with probability at least \math{1-e^{-m/8}}, \math{N_0\ge m/4}.
    Let \math{P_\alpha=\Prob[|C_i|\ge e^{\alpha}]}. Then,
    for \math{\alpha\ge 1},
    \math{
      P_\alpha=1-\frac{2}{\pi}\tan^{-1}(e^{\alpha})
    \ge 0.607 e^{-\alpha}.}
    Therefore, for  \math{\alpha\ge 1},
    using the Chernoff bound in Lemma~\ref{lemma:chernoff},
    \math{\Prob[N_\alpha\le mP_\alpha/2]\le \exp(-mP_\alpha/8)
      \le\exp(-me^{-\alpha}/14)}. Using a union bound for
    \math{\alpha\in\{0,\ldots,\ell\}},
    with probability at least \math{1-e^{-m/8}-\sum_{\alpha\in[\ell]}e^{-me^{-\alpha}/14}},
    \mld{N_0\ge m/4
      \qquad\text{and}\qquad
      N_\alpha\ge mP_\alpha/2\ge 0.304\cdot me^{-\alpha}
      \qquad\text{for \math{\alpha\in[\ell]}}.}
    Using~\r{eq:lower-cauchy-1}, with this same probability,
    \mld{
      \sum_{i\in[m]}|C_i|\ge \frac{m}{4}+0.1915 \ell\ge  \frac{m}{4}
      \left(1+\frac34 \ell\right).
      \label{proof:lemma-tail-2}
    }
    Then the largest summand in the failure probability
    is \math{\exp(-me^{-\ell}/14)\le\exp(-c d\ln(td))} for
    \math{\ell\le\frac43 f(td)}. We therefore pick \math{\ell=\floor{\frac43 f(td)}}. 
    All other
    summands in the failure probability
    for \math{\alpha=\ell-1,\ell-2,\ldots}  are decaying
    at least geometrically with
    ratio at most \math{\frac12} for \math{c\ge1, d\ge2}, so the full sum is at most twice the
    largest term. Using \r{proof:lemma-tail-2},
    with probability at least
    \math{1-2\exp(-c d\ln (td))},
    \mld{\sum_{i\in[m]}|C_i|\ge\frac{1}{4}m
      \left(1+\frac34\floor{\frac43 f(td)}\right).}
  \end{proof}

\subsection{Proof of Theorem~\ref{theorem:L1-subspace}: \math{\ell_1}-Embedding}
Fix a constant \math{t\ge 4} which will control failure probabilities in the
algorithms, and define
the embedding of \math{\matA} to \math{\Atilde\in\R^{d\times r}} by
\mld{
  \Atilde=
  \left[
  \begin{matrix}
    \sqrt{d}\ln(td)\cdot\Atilde_1\\
    \Atilde_2
  \end{matrix}
  \right]
  =
  \left[
    \begin{matrix}
    \sqrt{d}\ln (td)\cdot(\matA\Pi_1\Pi_1\transp\matA)^{1/2}\\
    \Pi_2\transp\matA
    \end{matrix}
    \right],
  \label{eq:embedding-def}
}
where \math{\Atilde_1} is our black-box \math{\ell_2}-subspace-embedding (any
\math{\ell_2}-subspace-embedding will do) and \math{\Pi_2\transp\in\R^{r\times n}}
is
a random
``spreading'' operator. For the rest of this section, we condition on the
event:
\mld{\text{for all \math{\xx\in\R^{d}},}\qquad
  {\textstyle\frac12}\norm{\matA\xx}_2
  \le
  \norm{\Atilde_1\xx}_2
  \le
  {\textstyle\frac32}\norm{\matA\xx}_2,
}
as can be arranged with probability at least \math{1-1/t} by
suitable choice of a distribution for \math{\Pi_1}.
Let  \math{\yy\in\R^{n}} be any vector that is
\emph{independent} of \math{\Pi_2}.
The spreading operator
\math{\Pi_2} is  drawn from a distribution 
that satisfies.
\mld{
  \Pi_2\transp\yy
  \sim
  \left[
    \begin{matrix}
      \gamma_1 \rvX_1\\
      \gamma_2 \rvX_2\\
      \vdots\\
      \gamma_r \rvX_r
    \end{matrix}
    \right],
  \label{eq:spread1}
  }
where \math{\gamma_j=\sum_{i\in S_j}|y_i|}, \math{\rvX_j} are standard
Cauchy random variables (for \math{\ell_p}-regression, replace with
\math{p}-stable random variables, for \math{1\le p< 2}), and
\math{S_1,\ldots,S_r} are a random partition of
\math{[n]} obtained by independently hashing each
\math{i\in[n]} randomly into one of the sets \math{S_1,\ldots,S_r}.
Observe that \math{\norm{\yy}_1} is preserved in \math{\Pi_2\transp\yy} through
the \math{\gamma}s because \math{\sum_{j}\gamma_j=\sum_{j}\sum_{i\in S_j}|y_i|
  =\sum_{i\in [n]}|y_i|=\norm{\yy}_1}.
One way to realize \math{\Pi} is using
\math{n} independent standard Cauchy random variables,
\math{\rvC_1,\ldots,\rvC_n}:
\mld{
  \Pi_2\transp\yy
  \sim
  \left[
    \begin{matrix}
      \sum_{i\in S_1} y_i\rvC_i\\
      \sum_{i\in S_2} y_i\rvC_i\\
      \vdots\\
      \sum_{i\in S_r} y_i\rvC_i
    \end{matrix}
    \right],\label{eq:spread2}
}
where \math{\Pi_2\transp} is a product of a sparse matrix and a diagonal matrix
of Cauchy's, realized by the
transform proposed in
\cite[page 21, \math{\Pi_2\transp=\phi D}]{WW2018}.
Using standard properties of the Cauchy distribution, one can verify that
\r{eq:spread2} satisfies \r{eq:spread1}. One can also directly implement
\r{eq:spread1} using just \math{r} independent Cauchy's. To apply
\math{\Pi_2} to a matrix \math{\matA}, we apply it to each
column of \math{\matA}. Using the transform in \cite{WW2018}
which gives \r{eq:spread2},
\math{\Pi_2\transp\matA} has dependent columns.

An Auerbach basis for
  \math{\matA} \cite{A1930}, is a basis \math{\matU}
  for the columns in \math{\matA} which satisfies:
  \mld{
  \norm{\matU}_1\le d,
  \qquad
  \forall\xx\in\R^{d},
  \norm{\matU\xx}_1\ge\norm{\xx}_\infty,
  \qquad
  \text{and}
  \qquad
  \matA=\matU\matR.
  }
An Auerbach basis for
\math{\matA} will be useful in proving that \math{\Atilde} does not dilate
the vectors too much in \math{1}-norm. We do not explicitly construct such a
basis, we just need that one exists.
\begin{lemma}\label{lemma:upper} Let \math{\matU} be a basis for the
  range of \math{\matA}.
  For all \math{\xx\in\R^{d}}
  \mld{
    \norm{\Atilde\xx}_1\le({\textstyle\frac32}d\ln(td)+\norm{\Pi_2\transp\matU}_1)
    \cdot
    \norm{\matA\xx}_1.
    }
\end{lemma}
\begin{proof}
  \math{\norm{\Atilde\xx}_1=\sqrt{d}\ln (td)\norm{\Atilde_1\xx}_1+\norm{\Atilde_2\xx}_1}.
  We bound each term separately.
  \eqar{
    \norm{\Atilde_1\xx}_1
    &\le&
    \sqrt{d}\norm{\Atilde_1\xx}_2\nonumber\\
    &\le&
    {\textstyle\frac32}\sqrt{d}\norm{\matA\xx}_2\nonumber\\
    &\le&
    {\textstyle\frac32}\sqrt{d}\norm{\matA\xx}_1.\label{eq:upper1}
    \\[10pt]
    \norm{\Atilde_2\xx}_1
    &=&
    \norm{\Pi_2\transp\matA\xx}_1\nonumber\\
    &=&
    \norm{\Pi_2\transp\matU\matR\xx}_1\nonumber\\
    &\le&
    \norm{\Pi_2\transp\matU}_1\norm{\matR\xx}_\infty\nonumber\\
    &\le&
    \norm{\Pi_2\transp\matU}_1\norm{\matU\matR\xx}_1\nonumber\\
    &=&
    \norm{\Pi_2\transp\matU}_1\norm{\matA\xx}_1.\label{eq:upper2}
  }
Combining \r{eq:upper1} and \r{eq:upper2} proves the lemma.
\end{proof}

We now analyze properties of \math{\Pi_2} which are useful for constructing an
approximation to \math{\matU}.
\begin{lemma}\label{lemma:upper-Pi2U}
  For \math{t\ge 2} and \math{rd>16},
  with probability at least \math{1-1/t},
  \math{\norm{\Pi_2\transp\matU}_1\le t\ln(trd)\norm{\matU}_1}.
\end{lemma}
\begin{proof} By the spreading property of \math{\Pi_2},
  \mld{
    \Pi_2\transp\matU=
    \left[
      \begin{matrix}
        \Pi_2\matU^{(1)}&\Pi_2\matU^{(1)}&\cdots&\Pi_2\matU^{(d)}
      \end{matrix}
      \right]
    \sim
    \left[
    \begin{matrix}
      \gamma_{11}\rvX_{11}&\gamma_{12}\rvX_{12}&\cdots&\gamma_{1d}\rvX_{1d}\\
      \gamma_{21}\rvX_{21}&\gamma_{22}\rvX_{22}&\cdots&\gamma_{2d}\rvX_{2d}\\
      \vdots&\vdots&&\vdots\\
      \gamma_{r1}\rvX_{r1}&\gamma_{r2}\rvX_{r2}&\cdots&\gamma_{rd}\rvX_{rd}\\
    \end{matrix}
    \right],
  }
    where
    \math{\rvX_{1j},\rvX_{2j},\ldots,\rvX_{rj}} are independent
    Cauchy random variables, \math{\rvX_{ij}} and \math{\rvX_{ik}}
    are dependent, and
    \math{\sum_{i=1}^r\gamma_{ij}=\norm{U^{(j)}}_1}. Therefore,
    \mld{\norm{\Pi_2\transp\matU}_1=
      \sum_{i\in[r],j\in[d]}\gamma_{ij}|\rvX_{ij}|}
    is a weighted sum of the sizes of \math{rd} Cauchy's
    (not-necessarily independent) with sum of weights equal to
    \math{\norm{\matU}_1}.
    Lemma~3 of \cite{malik209} gives a bound for the upper tail of
    a sum of Cauchy's \math{\rvX_i}:
    \mld{
      \Prob\left[\sum_{i\in S}|\gamma_i\rvX_i|\ge t\sum_{i\in S}|\gamma_i|\right]\le
      \frac{1}{\pi t}\left(\frac{\ln(1+(2t|S|)^2)}{1-1/(\pi t)}+1\right)
      \ {\buildrel (a)\over\le}\ 
      \frac{\ln( t |S|)}{t},
    }
    where the last inequality in (a) holds for \math{t\ge 2} and \math{|S|>16}.
    In our case,
    \math{\sum_{i\in S}|\gamma_i|=\norm{U}_1} and
    \math{|S|=rd}. Since \math{r\ge d}, so, for \math{rd>16},
    \mld{
      \Prob[\norm{\Pi_2\transp\matU}_1\ge t\ln(trd)\norm{\matU}_1]
      \le
      \frac1{t}.
      }    
\end{proof}
Let \math{rd>16}.
Combining Lemma~\ref{lemma:upper} with Lemma~\ref{lemma:upper-Pi2U} and
using \math{\norm{\matU}_1\le d}, we have:
\begin{theorem}[Bounded Dilation]\label{theorem:dilation}
 Let \math{t\ge 2} and \math{rd>16}. With probability at least \math{1-1/t}, 
for all \math{\xx\in\R^d},
 \mld{
\norm{\Atilde\xx}_1\le({\textstyle\frac32}d\ln (td)+td\ln(trd))
    \cdot
    \norm{\matA\xx}_1.
}
\end{theorem}

We now show that \math{\norm{\Atilde\xx}_1} is not much smaller than
\math{\norm{\matA\xx}_1} for every \math{\xx\in\R^d}. The proof breaks down
into two cases,
\math{\norm{\matA\xx}_2\ge\norm{\matA\xx}_1/\beta}
(the sparse case, see Lemma~\ref{lemma:sparse}) and
\math{\norm{\matA\xx}_2<\norm{\matA\xx}_1/\beta} (the dense case, see
Lemma~\ref{lemma:dense-y}). The sparsity is controlled by the parameter
\math{\beta}, which
we will fix later.
\begin{lemma}\label{lemma:sparse}
  If \math{\norm{\matA\xx}_2\ge\norm{\matA\xx}_1/\beta}, then
  \mld{\norm{\Atilde\xx}_1\ge
    \frac{\sqrt{d}\ln (td)}{2\beta}\norm{\matA\xx}_1.}
\end{lemma}
\begin{proof}
  \math{\norm{\Atilde\xx}_1=\sqrt{d}\ln d\norm{\Atilde_1\xx}_1+\norm{\Atilde_2\xx}_1\ge \sqrt{d}\ln d\norm{\Atilde_1\xx}_1\ge
    \sqrt{d}\ln d\norm{\Atilde_1\xx}_2\ge{\textstyle\frac12}\sqrt{d}\ln d
    \norm{\matA\xx}_2}. The lemma now follows from the assumed condition
  on \math{\norm{\matA\xx}_2}.
\end{proof}
If a vector is dense, we show that \math{\Pi_2} does not shrink its
\math{\ell_1}-norm
too much.
\begin{lemma}\label{lemma:dense-y}
  Let \math{\yy=\matA\xx} and
  suppose \math{\norm{\yy}_2<\norm{\yy}_1/\beta}.
  If \math{r\ge (14 cd\ln(td))\cdot e^{-\frac43 f(td)}}, then 
  \mld{\Prob\left[\norm{\Pi_2\transp\yy}\ge \frac{1+\frac34\floor{\frac43 f(td)}}{16}\norm{\yy}\right]\ge
    \textstyle
    1-2\exp(-\frac12 cd\ln(td))-\exp\left(-\frac{1}{16}\beta^2+\frac{1}{2}r(1+\ln 2)\right)
    \label{eq:lem-dense-2}
  }
\end{lemma}
\begin{proof}
By homogeneity, we may assume \math{\norm{\yy}_1=1} and hence
    \math{\norm{\yy}_2<1/\beta}. By the spreading property of
    \math{\Pi_2\transp},
    \math{\norm{\Pi_2\transp\yy}_1\sim\sum_{j\in[r]}\gamma_j|\rvC_j|}, where
    \math{\rvC_j} are independent Cauchys and the \math{\gamma_j}
    (which are independent of the \math{\rvC_j})  satisfy the
    assumptions of the negative independence lemma, Lemma~\ref{lemma:negative}.
    Let \math{\cl E} be the event that at least
    \math{r/2} of the \math{\gamma_j} are larger than \math{1/2r}.
    By part 3 of the negative independence lemma, with \math{k=2},
    \mld{\Prob[\cl E]\ge
      1-\exp\left(-\frac{\beta^2}{16}+\frac{r}{2}(1+\ln 2)\right).}
    Conditioning on this event
    \math{\cl E},
    let \math{I} be the indices \math{j} for which
    \math{\gamma_j>1/2r}, \math{|I|\ge r/2}. Then,
    \mld{\sum_{j\in[r]}\gamma_j|\rvC_j|>\frac{1}{2r}\sum_{j\in I}|\rvC_j|.
    \label{eq:dense1}}
    Using the lower tail inequality for independent Cauchys,
    Lemma~\ref{lemma:lower-cauchy-2}, with
    \mld{m=|I|\ge \frac{r}{2}\ge (14\cdot {\textstyle\frac12}d\ln(td))\cdot
      e^{\frac43 f(td)},}
    we have that,
    with probability at least \math{1-2e^{-\frac12 cd\ln(td)}},
    \mld{\textstyle\sum_{j\in I}|\rvC_j|> \frac14 m(1+\frac34\floor{\frac43 f(td)})\ge
      \frac18 r(1+\frac34\floor{\frac43 f(td)}).}
      Using~\r{eq:dense1},
      \mld{\Prob\left[\sum_{j\in[r]}\gamma_j|\rvC_j|\le
         { \textstyle\frac1{16} (1+\frac34\floor{\frac43 f(td)})}\right]
      \ \le\
      \Prob\left[\sum_{j\in I}|\rvC_j|\le { \textstyle\frac18 r(1+\frac34\floor{\frac43 f(td)})}\right]
      \le 2e^{-\frac12 cd\ln(td)}.
    \label{eq:proof-dense-1}}
      To conclude, condition on \math{\cl E}  and
      use Lemma~\ref{lemma:conditional}.
\remove{    \eqan{
      \Prob\left[\norm{\Pi_2\transp\yy}_1> \frac{f(td)}{16}\right]
      &\ge&
      1-\Prob\left[\sum_{j\in[r]}\gamma_j|\rvC_j|\le \frac{f(td)}{16}
        \ \Bigl|\  \cl E\right] - \Prob[\overline{\cl E}]\\
      &\ge&\textstyle
      1-2\exp(-\frac12 cd\ln(td))-\exp\left(-\frac{1}{16}\beta^2+\frac{1}{2}r(1+\ln 2)\right),
    }
    which proves the lemma.}
    \end{proof}
Lemma~\ref{lemma:dense-y} only applies for a single fixed \math{\yy} that is
independent of \math{\Pi_2\transp}. We need a result which applies to any
\math{\yy}. To do so, we need to use a \math{\delta}-net for the range of
\math{\matA}. Specifically, define the ball \math{B} of 1-norm radius
\math{1} in the range of \math{\matA} as
\mld{
  B=\{\yy\mid \yy=\matA\xx, \norm{\yy}_1=1\}.}
The set \math{\cl N=\{\zz_1,\ldots,\zz_K\}\subset B} is a \math{\delta}-net for
\math{B} if for every \math{\yy\in B}, there is a \math{\zz\in\cl N} for which
\math{\norm{\yy-\zz}_1\le\delta}. In \cite{BLM1989}, it is shown that for any
fixed \math{\matA}, there is a \math{\delta}-net of size
\math{|\cl N|=K\le (3/\delta)^d}. By applying the union bound to Lemma~\ref{lemma:dense-y},
\begin{lemma}\label{lemma:dense-all-y}
  Let \math{\cl N} be a \math{\delta}-net for \math{B} with 
  \math{|\cl N|\le (3/\delta)^d}.
  If \math{r\ge (14 cd\ln(td))\cdot e^{\frac43 f(td)}}, then, with probability at least 
  \math{1-2\exp\left(-\frac12 cd\ln(td)+d\ln(\frac{3}{\delta})\right)-
    \exp\left(-\frac{1}{16}\beta^2+\frac12 r(1+\ln 2)+d\ln(\frac{3}{\delta})\right)}, for all \math{\zz_i\in\cl N},
    \mld{
      \text{either } \qquad\norm{\zz_i}_2\ge
      \frac{1}{\beta},
      \qquad
      \text{or } \qquad
      \norm{\Pi_2\transp\zz_i}_1\ge\frac{(1+\frac34\floor{\frac43 f(td)})}{16}.
    }
\end{lemma}
We are now ready to prove the main contraction bound.
    \begin{theorem}\label{theorem:contraction}
      Assume bounded dilation,
      for all \math{\xx\in\R^d}:
      \math{\norm{\Atilde\xx}_1\le\kappa\norm{\matA\xx}_1}.
      For
      \math{\beta\ge \frac{8\sqrt{d}\ln(td)}{1+\frac34\floor{\frac43 f(td)}}}
        and
      \math{r\ge (14 cd\ln d)\cdot e^{\frac43f(td)}},
      set the size of the
      \math{\delta}-net to 
  \mld{\delta=\frac{\sqrt{d}\ln(td)}{4\kappa\beta}.}
  Then, with probability at least 
  \math{1-2\exp\left(-\frac12 cd\ln(td)+d\ln(\frac{3}{\delta})\right)-
    \exp\left(-\frac{1}{16}\beta^2+\frac12 r(1+\ln 2)+d\ln(\frac{3}{\delta})\right)}, 
  \mld{\text{for all \math{\xx\in\R^d}:}\qquad\norm{\Atilde\xx}_1
    \ge \frac{\sqrt{d}\ln(td)}{4\beta}
    \norm{\matA\xx}_1.}
  \end{theorem}
    \begin{proof}
  Let \math{\yy=\matA\xx}. By homogeneity, we may assume that
  \math{\norm{\yy}_1=1}.
    Let
  \math{\cl N} be the \math{\delta}-net for the range of~\math{\matA}
  from Lemma~\ref{lemma:dense-all-y}, \math{|\cl N|\le (3/\delta)^d}.
  With probability as in the theorem statement, Lemma~\ref{lemma:dense-all-y}
      holds.
      Let
      \math{\yy=\zz+\yy-\zz}, where \math{\zz\in\cl N} and
        \math{\norm{\yy-\zz}_1\le\delta}.
  Since \math{\zz} is in the range of \math{\matA}, so is \math{\yy-\zz}, and so
  we may write \math{\zz=\matA\xx_1} and \math{\yy-\zz=\matA\xx_2},
  and \math{\yy=\matA\xx_1+\matA\xx_2}.
  Since \math{\matA} has full rank, \math{\xx=\xx_1+\xx_2}.

  Let us now consider \math{\Atilde\xx=\Atilde\xx_1+\Atilde\xx_2}. There
  are two cases. First,
  if \math{\norm{\matA\xx_1}_2=\norm{\zz}_2\ge 1/\beta},
  we apply Lemma~\ref{lemma:sparse} to conclude
  \mld{
    \norm{\Atilde\xx_1}_1\ge\frac{\sqrt{d}\ln(td)}{2\beta}.}
  Second, if \math{\norm{\matA\xx_1}_2=\norm{\zz}_2< 1/\beta}, then
  \math{\norm{\Atilde\xx_1}_1\ge \norm{\Pi_2\transp\zz}_1\ge f(td)/16}
  (by Lemma~\ref{lemma:dense-all-y}). We conclude that
  \mld{
    \norm{\Atilde\xx_1}_1\ge
    \min\left(\frac{1+\frac34\floor{\frac43 f(td)}}{16},\frac{\sqrt{d}\ln(td)}{2\beta}\right).
  }
  By the triangle inequality,
  \math{\norm{\Atilde\xx}_1\ge\norm{\Atilde\xx_1}_1-\norm{\Atilde\xx_2}_1}.
  Since \math{\norm{\Atilde\xx_2}_1\le\kappa\norm{\matA\xx_2}_1
    =\kappa\norm{\yy-\zz}_1\le\kappa\delta}, we have
  \mld{
    \norm{\Atilde\xx}_1\ge
    \min\left(\frac{1+\frac34\floor{\frac43 f(td)}}{16},\frac{\sqrt{d}\ln(td)}{2\beta}\right)-\kappa\delta.
    }
  The theorem follows from the choice of \math{\delta} and
  \math{\beta\ge \frac{8\sqrt{d}\ln(td)}{1+\frac34\floor{\frac43 f(td)}}}
\end{proof}
    We are now ready to prove the main theorem, for which we fix the parameters
    \math{t,r,\beta,\delta}, etc. Let \math{d\ge 4} and
    \math{t\ge 10}. The dilation
    \math{\kappa} is given
    in Theorem~\ref{theorem:dilation} and
    the \math{\delta}-net size is defined in
    Theorem~\ref{theorem:contraction}, where
    by setting 
    \math{\exp\left(-\frac12 cd\ln(td)+d\ln(\frac{3}{\delta})\right)
      \le \frac{1}{4t}} and
    \math{\exp\left(-\frac{1}{16}\beta^2+\frac12 r(1+\ln 2)+d\ln(\frac{3}{\delta})\right)\le \frac{1}{2t}}, the total failure probability is at most
    \math{\frac{1}{t}}. Also, in Theorem~\ref{theorem:contraction},
    \math{\beta\ge (8\sqrt{d}\ln(td))/f(td)}.
    Collecting all the constraints together, we have:
    \mld{
      \renewcommand{\arraystretch}{1.5}
      \begin{array}{rcl}
        \kappa&=&\frac32d\ln (td)+td\ln(trd)\\
        \delta&=&\frac{\sqrt{d}\ln(td)}{4\kappa\beta}\\
        \beta&\ge& \frac{8\sqrt{d}\ln(td)}{1+\frac34\floor{\frac43 f(td)}}\\
        r&=& (14 cd\ln (td))\cdot e^{\frac43 f(td)}\\
        r&\ge&(28\ln 4t+28d\ln(3/\delta))\cdot e^{\frac43 f(td)}\\
        \beta^2&\ge&16\ln 2t+8r(1+\ln 2)+16d\ln(3/\delta).
      \end{array}
      \label{eq:settings}
    }
Before we give the main contraction bound, we show how to pick the parameters
to satisfy \r{eq:settings}. The main challenge is to satisfy the inequality
constraints for \math{r} and \math{\beta^2}. We consider two cases
separately, \math{f(x)=0} and \math{f(x)\in O(\ln\ln(x))}.
\begin{lemma}\label{lemma:choose-parameters}
  Suppose \math{f(x)=0} and \math{d,t\ge 10}. Then, the
  following choices for \math{r} and \math{\beta} satisfy
  the requirements in 
  \r{eq:settings}:
  \mld{
    \begin{array}{rcl}
      r&=&80 d\ln(td)\\
       \beta&=&16\sqrt{d}\ln(td).
    \end{array}
  }
\end{lemma}
\begin{proof}
  First, we use a crude bound on \math{r} to
  bound \math{\kappa}. Since \math{d,t\ge 10}:
  \math{
    r\le 8\cdot td\ln(td)\le (td)^2.
  }
  Using this crude bound,
  \mld{
    \kappa
    =
    \frac32\cdot d\ln (td)+td\ln(trd)
    \le
    \frac3{20}\cdot td\ln (td)+td\ln(td)^3
    =
    \frac{63}{20}\cdot td\ln (td).\label{eq:proof-constraints-1}
  }
  By the choice of
  \math{\beta}, \math{\delta=1/64\kappa}, hence
  \math{3/\delta\le 605\cdot td\ln (td)}. Therefore,
  \mld{28\ln 4t+28d\ln(3/\delta)\le 28\ln 4t+28d\ln(605\cdot td\ln (td))
    \le
    2.8\cdot d\ln(2u/5)+28d\ln(605u\ln u),}
  where \math{u=td}. By straightforward calculus,
  \math{2.8\ln(2u/5)+28\ln(605u\ln u)\le 80\ln u}, which proves
  \mld{
    r=80d\ln u\ge 2.8d\ln(2u/5)+28d\ln(605u\ln u)\ge
    28\ln 4t+28d\ln(3/\delta).
    }
  We now show the bound on \math{\beta^2=256d\ln^2 u}, where \math{u=td}.
  Using \math{d,t\ge 10} and the definition of \math{r},
  \eqan{
    16\ln 2t+8r(1+\ln 2)+16d\ln(3/\delta)
    &\le&
    16\ln 2t+14r+16d\ln(3/\delta)
    \\
    &\le&
    1.6\cdot d\ln(u/5)+1120\cdot d\ln u+16d\ln(605u\ln u),
  }
  where we used the bound on \math{3/\delta} from earlier.
  Again, by straightforward calculus, for \math{u\ge 100},
  one can show that the RHS is at most 
  \math{256d\ln^2 u},
  which proves the lower bound for \math{\beta^2}.
\end{proof}

\begin{lemma}\label{lemma:choose-parameters-2}
  For \math{q\ge 3},
  suppose \math{f(x)=\frac{3}{4q}\ln\ln x} and \math{d,t\ge 10} also satisfy
  \math{td \ge q^{1.17}}. Then, the
  following choices for \math{r} and \math{\beta} satisfy
  the requirements in 
  \r{eq:settings}:
  \mld{
    \renewcommand{\arraystretch}{1.5}
    \begin{array}{rcl}
      r&=&100 d\ln^{1+1/q}(td)\\
       \beta&=&\displaystyle\frac{16q\sqrt{d}\ln(td)}{\ln\ln(td)}.
    \end{array}
  }
\end{lemma}
\begin{proof}
  Again, since \math{d,t\ge 10}, the same crude bound on \math{r} holds,
  \math{
    r\le 10\cdot td\ln^{1+1/q}(td)\le (td)^2
  } (because \math{q\ge 3}).
  Hence, we get the same bound
  for \math{\kappa}, \math{\kappa\le\frac{63}{20}\cdot td\ln (td).} Let \math{u=td}.
  Then,
  \mld{
    \frac{3}{\delta}
    =
    \frac{12\kappa\beta}{\sqrt{d}\ln(td)}
    \le
    605 q\frac{u\ln u}{\ln\ln u}.
  }
  We have that
  \eqan{
    &&r-(28\ln 4t+28d\ln(3/\delta))\cdot e^{\frac43 f(td)}\\
    &\ge&
    100d\ln^{1+1/q}u-
    \left(2.8\cdot d\ln(2u/5)+28d\ln\left(605 q\frac{u\ln u}{\ln\ln u}\right)
    \right)\ln^{1/q}u
    \\
    &=&
    d\ln^{1+1/q}u\left(100-\frac{2.8\ln(2u/5)+28\ln\left(605 \frac{u\ln u}{\ln\ln u}\right)}{\ln u}
    -
    28\frac{\ln q}{\ln u}\right)
    \\
    &\ge&
    d\ln^{1+1/q}u\left(24.1-28\frac{\ln q}{\ln u}\right).
  }
  Since \math{u\ge q^{1.17}}, the RHS is positive, proving the bound
  on \math{r}.
    Also, since
  \math{1+\frac34\floor{\frac43 f(td)}\ge f(td)},
  \mld{
    \frac{8\sqrt{d}\ln(td)}{1+\frac34\floor{\frac43 f(td)}}
    \le
    \frac{8\sqrt{d}\ln(td)}{\frac{3}{4q}\ln\ln(td)}
    \le
    \frac{16q\sqrt{d}\ln(td)}{\ln\ln(td)},
  }
  which is the first bound on \math{\beta}. For the second bound on
  \math{\beta}, using \math{q\ge 3},
  \math{\beta^2\ge(2304 d \ln^2(u))/(\ln\ln u)^2}, and
  \eqan{
    &&16\ln 2t+8r(1+\ln 2)+16d\ln(3/\delta)\\
    &\le&
    16\ln (u/5)+1355 d\ln^{1+1/q}u
    +
    16 d\ln\left(605 q\frac{u\ln u}{\ln\ln u}\right)
    \\
    &\le&
    1.6d\ln (u/5)+1355 d\ln^{4/3}u
    +
    16 d\ln\left(605 \frac{u\ln u}{\ln\ln u}\right)
    +
    13.68 d\ln u, 
  }
  where the last step uses \math{u\ge q^{1.17}}, \math{q\ge 3}
  and \math{d,t\ge 10}. 
  Therefore,
  \eqan{
    &&\beta^2-(16\ln 2t+8r(1+\ln 2)+16d\ln(3/\delta))\\
    &\ge&
    d\left(
    \frac{2304 \ln^2(u)}{\ln^2\ln u}
    -1.6\ln (u/5)
    -1355 \ln^{4/3}u
    -
    16 \ln\left(605 \frac{u\ln u}{\ln\ln u}\right)
    -
    13.68 \ln u
    \right)
  }
  The expression in parentheses is just a function of \math{u} growing
  asymptotically as \math{\ln^2(u)/\ln^2\ln u}. By straightforward calculus,
  this term is positive for \math{u\ge 100}.
\end{proof}
Lemmas~\ref{lemma:choose-parameters} and ~\ref{lemma:choose-parameters-2}
give sufficient conditions for satisfying all
the requirements in \r{eq:settings}. The requirements in  \r{eq:settings}
are enough to ensure that Theorems~\ref{theorem:dilation}
and~\ref{theorem:contraction} each hold with probability
at least \math{1-1/t}, so by a union bound, both theorems hold with
probability at least \math{1-2/t}. We therefore get the following theorem. 
\begin{theorem}[\math{\ell_1}-embedding]\label{theorem:main-embedding}
  Let \math{d,t\ge 10}. With probability at least
  \math{1-2/t}, for all \math{\xx\in\R^{d}},
  \begin{enumerate}[label={(\roman*)}]
  \item If \math{r\ge 80 d\ln (dt)}, the distortion is in
    \math{\Theta(d\ln d)}. Specifically,
      \mld{
    \frac{1}{64}\cdot \norm{\matA\xx}_1\le\norm{\Atilde\xx}_1\le
    \frac{63}{20}\cdot td\ln(td)\cdot \norm{\matA\xx}_1.
  }
    \item If \math{r\ge 100 d\ln^{1+1/q} (td)}, for \math{q\ge 3},
      the distortion is in
      \math{\Theta((d\ln d)/\ln\ln d)}.  Specifically, for \math{td\ge q^{1.17}},
      \mld{
    \frac{\ln\ln(td)}{64q}\cdot \norm{\matA\xx}_1\le\norm{\Atilde\xx}_1\le
    \frac{63}{20}\cdot td\ln(td)\cdot \norm{\matA\xx}_1.
       }
  \end{enumerate}
\end{theorem}
\begin{proof}
  Apply a union bound to Theorems~\ref{theorem:contraction} (lower bound)
  Theorem~\ref{theorem:dilation} (upper bound)
  with Lemmas~\ref{lemma:choose-parameters}
  and~\ref{lemma:choose-parameters-2}.
  To get the specific upper bound, 
  use the bound
  on \math{\kappa} from \r{eq:proof-constraints-1}).
\end{proof}

\subsection{Proof of Theorem~\ref{theorem:l1-sampling}: \math{\ell_1} Sampling Based Embedding}
\label{proof:l1-sampling}
Using Lemma~12 in\cite{malik186} with \math{r_2=15\log(6n/\delta)}
gives,
with probability at least
\math{1-\delta/3} and for all \math{i\in[n]},
\mld{
  \frac12\norm{\matU_{(i)}}_1\le\lambda_i\le\frac32\norm{\matU_{(i)}}_1.
}
which means that \math{\lambda_i/\sum_{i\in[n]}\lambda_i\ge\frac13
  \norm{\matU_{(i)}}_1/\norm{\matU}_1} for all \math{i\in[n]}.
Condition on this event holding.
Let \math{\matD} be a diagonal matrix whose diagonal entries are chosen
independently as follows:
\mld{
  \matD_{ii}=
  \begin{cases}
    1/{p_i}&\text{with probability \math{p_i}}\\
      0&\text{otherwise}
  \end{cases}
}
Then \math{\norm{\Atilde\xx}_1=\norm{\matD\matA\xx}_1}, for
any \math{\xx\in\R^{d}}.
Since \math{\matU} is a basis, \math{\matA=\matU\matR}, hence
\math{\norm{\matD\matA\xx}_1=
  \norm{\matD\matU\zz}_1}, where \math{\zz=\matR\xx}. The sampling
lemma follows if we prove, for every \math{\xx\in\R^{d}}
\mld{
  (1-\varepsilon)  \norm{\matU\xx}_1\le \norm{\matD\matU\xx}_1\le
  (1+\varepsilon)  \norm{\matU\xx}_1
}
A slightly modified version of the \math{\ell_1} sampling lemma,
\cite[Lemma 5]{malik186} gives that for a
a fixed \math{\xx} and \math{\varepsilon\le1}
\mld{\norm{\matD\matU\xx}_1\le
  (1+\varepsilon/2)  \norm{\matU\xx},\label{eq:app:upper1}
}
  with a failure probability at most
  \math{\exp(-s\varepsilon^2/28\alpha)}.
  Let \math{B} be the \math{d}-dimensional unit \math{\ell_1}-ball in the
  range of \math{\matU},
  \math{B=\{\zz=\matU\xx|\norm{\zz}_1=1\}}. Let
  \math{\cl N=\{\zz_1,\ldots,\zz_K\}\subset B}
  be a \math{\frac16\varepsilon}-net for
  of size \math{|\cl N|\le\left(\frac{18}{\varepsilon}\right)^d}~\cite{BLM1989}.
  By a union bound, with probability at least
  \math{1-\exp(-\frac{s\varepsilon^2}{7\alpha}+d\ln \frac{18}{\epsilon})},~\r{eq:app:upper1}
  holds for every \math{\zz_i\in\cl N}. Since
  \math{s\ge 28\varepsilon^{-2}d\alpha(\ln(18/\varepsilon)+d^{-1}\ln(3/\delta))},
  the failure probability is at most \math{\delta/3}.
  Now consider any \math{\xx\in\R^{d}} and let \math{\zz=\matU\xx}.
  By homogeneity, we may assume
  that \math{\norm{\zz}_1=1}. We show that
  \math{\norm{\matD\zz}_1\le1+\varepsilon}. Let \math{\zz_*\in\cl N}
  be the closest point in \math{\cl N} to \math{\zz}, so
  \math{\norm{\zz-\zz_*}_1\le\varepsilon/3}. Then
  \math{\zz=\zz_*+\zz-\zz_*} and by the triangle inequality,
  \mld{
    \renewcommand{\arraystretch}{1.25}
        \begin{array}{rcl}
    \norm{\matD\zz}_1&\le& \norm{\matD\zz_*}_1+\norm{\matD(\zz-\zz_*)}_1\\
    &\le& (1+\varepsilon/2)+\norm{\zz-\zz_*}\norm{\matD\ww}_1\\
    &\le&(1+\varepsilon/2)+(\varepsilon/6)\norm{\matD\ww}_1,
    \end{array}
  }
  where \math{\ww=(\zz-\zz_*)/\norm{\zz-\zz_*}_1} is a vector
  in the range of
  \math{\matU}
  with unit \math{\ell_1}-norm. Iterating with
  \math{\ww},
  \mld{
    \renewcommand{\arraystretch}{1.5}
    \begin{array}{rcl}
    \norm{\matD\zz}_1&\le& (1+\varepsilon/2)(1+(\varepsilon/6)+
    (\varepsilon/6)^2+\cdots)\\
    &=&\displaystyle
    \frac{1+\varepsilon/2}{1-\varepsilon/6}\ \le\ 1+\varepsilon.
    \end{array}
    }
(The last inequality holds for \math{\varepsilon\le 1}.)
  Therefore, with probability at least \math{1-\delta/3},
  \math{\norm{\matD\matU\xx}_1\le (1+\varepsilon)\norm{\matU\xx}_1} for
  all \math{\xx\in\R^d}. We condition on this
  bound holding for all \math{\xx}.
  Again, using the lower bound side of the sampling lemma from
  \cite{malik186},
  \mld{\norm{\matD\matU\xx}_1\le
  (1+\varepsilon/2)  \norm{\matU\xx},\label{eq:app:upper1}
}
  with a failure probability at most
  \math{\exp(-s\varepsilon^2/28\alpha)}. Again applying a union bound
  over the \math{\frac16\varepsilon}-net, and for our choice of
  \math{s}, with probability at least \math{1-\delta/3},
  \math{\norm{\matD\zz_i}_1\ge (1-\varepsilon/2)} for all \math{\zz_i\in\cl N}.
  Now consider any \math{\xx}, \math{\zz=\matU\xx}. As before, pick
  \math{\zz_*\in\cl N} with \math{\zz=\zz_*+\zz-\zz_*} and
  \math{\norm{\zz-\zz_*}_1\le\varepsilon/6}. We show that
  \math{\norm{\matD\zz}_1\ge 1-\varepsilon}. Indeed, by the
  triangle inequality,
  \mld{
      \renewcommand{\arraystretch}{1.25}
    \begin{array}{rcl}  
    \norm{\matD\zz}_1&\ge&\norm{\matD\zz_*}_1-\norm{\matD(\zz-\zz_*)}_1\\
    &\ge& (1-\varepsilon/2)-(1+\varepsilon)\norm{\zz-\zz_*}_1,
    \end{array}
  }
  where the last inequality is because we conditioned on
  \math{\norm{\matD\zz}_1\le(1+\varepsilon)\norm{\zz}_1} for
  all \math{\zz} in the range of \math{\matU}. Since
  \math{\norm{\zz-\zz_*}_1\le\varepsilon/6}, 
  \math{\norm{\matD\zz}_1\ge1-\frac12\varepsilon-\frac16\varepsilon-
    \frac16\varepsilon^2\ge 1-\varepsilon}.

  Lastly, the embedding dimension \math{r} is sum of Bernoulli random
  variables with probability \math{p_i}, with
  \mld{\Exp[r]
    =\sum_{i\in[n]}p_i
    \le\sum_{i\in[n]}s\frac{\lambda_i}{\sum_{j\in[n]}\lambda_j}=s.
  }
  Therefore, by the Chernoff bound in Lemma~\ref{lemma:chernoff},
  \math{\Prob[r>2s]\le e^{-s/3}}.
  The full failure probability follows by applying
  a union bound to the three failure probabilities of \math{\delta/3}
  together with the failure probability for the embedding dimension.
  \qedsymb

\subsection{Proof of Lemma~\ref{lemma:constructing-U}: Constructing
A Well-Conditioned Basis}

The runtimes of all the algorithms follows from the runtimes of the
embeddings and the time to compute
\math{\matA\matR^{-1}\matC} from right to left. We prove the bounds on
\math{\alpha(\matU)}.

\paragraph{Part (\rn{1}).} 
\math{\matU=\matA\Atilde^{-1}}, where \math{\Atilde=(\matA\transp
  \Pi\Pi\transp\matA)^{1/2}}. 
Let \math{\matQ} be a basis for the range of
\math{\matA}, with \math{\matA=\matQ\matS}. 
We first bound
\math{\norm{\matU}_1}.
\eqar{
  \norm{\matU}_1
  &=&
  \norm{\matA\Atilde^{-1}}_1\nonumber\\
  &=&
  \sum_{i\in[d]}\norm{\matA(\Atilde^{-1})^{(i)}}_1\nonumber\\
  &=&
  \sum_{i\in[d]}\norm{\matQ\matS(\Atilde^{-1})^{(i)}}_1\nonumber\\
  &\le&
  \norm{\matQ}_1\sum_{i\in[d]}\norm{\matS(\Atilde^{-1})^{(i)}}_\infty\nonumber\\  
  &\le&
  d\norm{\matQ}_1\norm{\matS(\Atilde^{-1})}_\infty\nonumber\\  
  &\le&
  d\norm{\matQ}_1\norm{\matS(\Atilde^{-1})}_2\nonumber\\  
  &\le&
  d\norm{\matQ}_1\norm{\matS}_2\norm{\Atilde^{-1}}_2  
}
Using \r{eq:sig-courant-fisher} with \math{\varepsilon=\frac12},
\math{\tilde\sigma_d\ge\frac12\sigma_d}, which means that
\mld{\norm{\Atilde^{-1}}_2=1/\tilde\sigma_d\le \sqrt{2}/\sigma_d=\sqrt{2}\norm{\matA^{-1}}_2.}
Therefore, \math{\norm{\matU}_1\le
  \sqrt{2}d\norm{\matQ}_1\norm{\matS}_2\norm{\matA^{-1}}_2}.
We now give a lower bound on \math{\norm{\matU\xx}_1}.
\eqar{
  \norm{\matU\xx}_1
  &=&
  \norm{\matA\Atilde^{-1}\xx}_1\nonumber\\
  &\ge&
  \norm{\matA\Atilde^{-1}\xx}_2\nonumber\\
  &\buildrel (a) \over \ge&
       \sqrt{{\textstyle\frac23}}\cdot\norm{\Atilde\Atilde^{-1}\xx}_2\nonumber\\
  &=&
  \sqrt{{\textstyle\frac23}}\cdot\norm{\xx}_2\nonumber\\
  &\ge&
  \sqrt{{\textstyle\frac23}}\cdot\norm{\xx}_\infty.\label{eq:Ux-lower}
}
Therefore, from the definition of \math{\alpha(\matU)} in \r{eq:def-alp},
and using \r{eq:Ux-lower} with the upper bound on \math{\norm{\matU}_1},
\mld{
  \alpha(\matU)=
  \frac{\norm{\matU}_1}{\min_{\norm{\xx}_\infty=1}{\norm{\matU\xx}_1}}
  \le
  \frac{\sqrt{2}d\norm{\matQ}_1\norm{\matS}_2\norm{\matA^{-1}}_2}{\sqrt{\frac23}}
  \label{eq:alp-bound}
}
The bound in \r{eq:alp-bound} holds for any
\math{\matQ}. By homogeneity, we can choose
\math{\norm{\matQ}_1=d}, so we should pick the best possible such
\math{\matQ} that minimizes \math{\norm{\matS}_2}. This gives
\math{
  \alpha(\matU)\le \sqrt{3}d^2\norm{\matA^{-1}}_2\min_{\norm{\matQ}_1\le d}\norm{\matS}_2
=\sqrt{3}d^2\kappa_1(\matA)}.

\paragraph{Part (\rn{2}).} 
The result follows from a, by now, standard well-conditioned basis
construction given a \math{\text{poly}(d)}-distortion embedding.
Recall \math{\matU=\matA\matR^{-1}}, where \math{\Atilde=\matQ\matR} and
\math{\matQ} is orthogonal.
For the upper bound on
\math{\norm{\matU}_1}, we get
\eqar{
  \norm{\matU}_1
  &=&
  \norm{\matA\matR^{-1}}_1\nonumber\\
  &=&
  \sum_{i\in[d]}\norm{\matA(\matR^{-1})^{(i)}}_1\nonumber\\
  &\buildrel (a) \over \le&
  \sum_{i\in[d]}\norm{\Atilde(\matR^{-1})^{(i)}}_1\nonumber\\
 & =&
 \norm{\Atilde\matR^{-1}} _1\nonumber\\
  &=&
 \norm{\matQ\matR\matR^{-1}} _1\nonumber\\
  &=&
 \norm{\matQ}_1\nonumber\\
 &\buildrel (b) \over \le&
 \sqrt{rd}\cdot \norm{\matQ}_F\nonumber\\
 &\buildrel (c) \over =&
 \sqrt{rd}\cdot \sqrt{d}.\label{eq:standard-upper}
}
Above, in the derivation of \r{eq:standard-upper}
and below in the derivation of \r{eq:standard-lower}:
(a) uses \r{eq:getU-part2-isometry} which is assumed;
(b) is the standard relationship between \math{\ell_1} and \math{\ell_2}
entrywise-norms;
(c) is because \math{\matQ} is orthogonal, so \math{\norm{\matQ}_F=\sqrt{d}}
and \math{\norm{\matQ\xx}_2=\norm{\xx}_2}.
For a lower bound on \math{\norm{\matU\xx}_1}, we get
\mld{
  \norm{\matU\xx}_1
  =
  \norm{\matA\matR^{-1}\xx}_1
  \ {\buildrel (a) \over \ge}\ 
  \frac{1}{\kappa}\norm{\Atilde\matR^{-1}\xx}_1
  =
  \frac{1}{\kappa}\norm{\matQ\xx}_1
  \ {\buildrel (b) \over \ge}\ 
  \frac{1}{\kappa}\norm{\matQ\xx}_2
  \ {\buildrel (c) \over \ge}\ 
  \frac{1}{\kappa}\norm{\xx}_2
  \ge
  \frac{1}{\kappa}\norm{\xx}_\infty.\label{eq:standard-lower}
}
Dividing \r{eq:standard-upper} by \r{eq:standard-lower} and setting \math{\norm{\xx}_\infty=1}, we get
\math{\alpha(\matU)\le \kappa d\sqrt{r}}.

\paragraph{Part (\rn{3}).} 
The construction has three steps.
First, construct an embedding \math{\Atilde_1}
with distortion \math{\kappa\in O(d\ln d)}
with embedding dimension \math{r\in O(d\ln d)} as in
Theorem~\ref{theorem:L1-subspace}.
Second, using part (\rn{2}),
construct a well-conditioned basis \math{\matU_1} using \math{\Atilde_1}, and
then sample \math{O(d^{3.5}\ln^{1.5}d)} rows to get an embedding
\math{\Atilde_2} satisfying \r{eq:thm-L1-sampling-isometry} 
with \math{\varepsilon=\frac12} (constant distortion).
We now use the
methods in the proof of Theorem~5 in~\cite{DDHKM2009}.
Let \math{\Atilde_2=\tilde\matQ\matS}, where \math{\matQ} is an orthogonal
basis for the range of \math{\Atilde_2}.
The time to compute \math{\tilde\matQ} and \math{\matS} is
in \math{O(d^{5.5}\ln^{1.5}d)}.
Construct the quadratic form parameterized by the square invertible
matrix \math{\matG\in\R^{d\times d}}
for the John-ellipsoid of \math{\tilde\matQ}. Specifically,
let \math{B=\{\xx\in\R^d|\norm{\tilde\matQ_2\xx}_1\le 1}, and let
\math{\matG} be
such that if \math{\xx\transp\matG\transp\matG\xx\le 1}, then
\mld{
\xx\transp\matG\transp\matG\xx\le
\xx\transp\tilde\matQ\transp\tilde\matQ\xx\le
\sqrt{d}\cdot\xx\transp\matG\transp\matG\xx.
}
The time to compute \math{\matG} is in \math{rd^5\ln r} where \math{r\in
  O(d^{3.5}\ln^{1.5}d)} is the embedding dimension of
\math{\Atilde_2}, which is a runtime in \math{O(d^{8.5}\ln^{2.5}d)}.

Let 
\math{\tilde\matU=\tilde\matQ\matG^{-1}}.
Then, it is shown in Equation~(6) of
\cite{DDHKM2009} that for all \math{\xx\in\R^{d}},
\mld{
  \norm{\xx}_2\le\norm{\tilde\matU_2\xx}_1\le\sqrt{d}\cdot\norm{\xx}_2.
  \label{eq:isometry-Utilde}
}  
Let \math{\matU=\matA\matS^{-1}\matG^{-1}}. We bound \math{\alpha(\matU)}.
First, we bound \math{\norm{\matU}_1}:
\eqar{
  \norm{\matU}_1
  &=&
  \norm{\matA\matS^{-1}\matG^{-1}}_1\nonumber\\
  &=&
  \sum_{i\in[d]}\norm{\matA(\matS^{-1}\matG^{-1})^{(i)}}_1\nonumber\\
  &\buildrel (a)\over\le&
  c\sum_{i\in[d]}\norm{\Atilde_2(\matS^{-1}\matG^{-1})^{(i)}}_1\nonumber\\
  &=&
  c\norm{\Atilde_2\matS^{-1}\matG^{-1}}_1\nonumber\\
  &\buildrel (b)\over=&
  c\norm{\tilde\matU}_1\nonumber\\
  &=&
  c\sum_{i\in[d]}\norm{\tilde\matU\ee_i}_1\nonumber\\
  &\buildrel (c)\over\le&
  c\sqrt{d}\cdot\sum_{i\in[d]}\norm{\ee_i}_2\nonumber\\
  &=&
  c\cdot d^{1.5}.\label{eq:upper-ellipsoidal}
}
In the derivation of \r{eq:upper-ellipsoidal} above, and also for
the derivation of \r{eq:lower-ellipsoidal} below:
(a) is because, by construction, \math{\Atilde_2} is a constant
factor \math{\ell_1}-embedding for \math{\matA};
in (b) we used \math{\Atilde_2=\tilde\matQ} and
\math{\tilde\matU=\tilde\matQ\matG^{-1}};
in (c) we used \r{eq:isometry-Utilde}. We now
lower bound \math{\norm{\matU\xx}_1}:
\mld{
  \norm{\matU\xx}_1
  =
  \norm{\matA\matS^{-1}\matG^{-1}\xx}_1
  \ {\buildrel (a)\over\ge}\ 
  c\cdot\norm{\Atilde_2\matS^{-1}\matG^{-1}\xx}_1
  \ {\buildrel (b)\over=}\ 
  c\cdot\norm{\tilde\matU\xx}_1
  \ {\buildrel (c)\over\ge}\ 
  c\cdot\norm{\xx}_2
  \ge
  c\cdot\norm{\xx}_\infty.\label{eq:lower-ellipsoidal}
}
Dividing \r{eq:upper-ellipsoidal} by \r{eq:lower-ellipsoidal},
using the definition of
\math{\alpha(\matU)} and setting \math{\norm{\xx}_\infty=1} gives
\math{\alpha(\matU)\le d^{1.5}}.
\qedsymb

\subsection{Proof of Lemma~\ref{lemma:lewis-intro}:
  Sampling Using Lewis Weights}

The proof is entirely based on  
\cite{CP2015} and follows from the approximation guarantees provided by
approximating the leverage scores provided
in~Theorem~\ref{theorem:L2-leverage-intro} using the
methods from \cite{malik186}. First, we need the main Lewis weights
approximation lemma from~\cite{CP2015}.
For completeness, we give the details, which also serves to fix some constants,
and hence may be otherwise useful.
In step 3(a) of the algorithm, we approximate the leverage scores of
\math{\matW^{-1/2}\matA} using error parameter \math{\varepsilon=\frac14} to
choose the dimensions of \math{\matG} and \math{\Pi}. In~\cite{malik186} it
is shown that \math{|\tilde\tau_i-\tau_i|\le (\frac{\varepsilon}{1-\varepsilon}+2\varepsilon)\tau_i}, or
\mld{
  \frac{1}{6}\cdot\tau_i\le\tilde\tau_i\le \frac{11}{6}\cdot\tau_i,
}
which in the terminology of
\cite{CP2015} means that the \math{\tilde\tau_i} are a 6-approximation
of \math{\tau_i}, written \math{\tilde\tau_i\approx_{6}\tau_i}.
\begin{lemma}[See Proof of Lemma~2.4 in \cite{CP2015}]
  \label{lemma:lewis-weight-iterative}
  If the iterative algorithm to compute
  Lewis weights is run for \math{2\log_2\log_2n} steps with a 6-factor
  approximation of leverage scores in each iteration, then
  the resulting Lewis weights are a
  \math{6\cdot2^{1/\log_2n}}-factor approximation.
\end{lemma}
So the resulting Lewis weight estimates are approximately within
a factor of \math{6} from the true lewis weights. We now reformulate/apply
Theorem~2.3 in \cite{CP2015} to obtain a concrete coreset size that  gives
a relative error embedding.
\begin{theorem}[Theorem~2.3 of \cite{CP2015}]
  \label{theorem:lewis-general-beta}
  Let \math{w_i} be a \math{\beta}-approximation to the true Lewis
  weights \math{w_i^*} of a matrix \math{\matA}, that is:
  \math{
    \frac{1}{\beta}w_i^*\le w_i\le\beta w_i^*.
  }
  Use sampling probabilities \math{p_i=w_i/\sum_{j\in[n]}w_j} and
  construct a sampling matrix \math{\Pi\transp} with
  \math{r=2c\cdot \beta^2\varepsilon^{-2} d\ln(2c\beta^2 d\varepsilon^{-2})} 
  rows, where each row of \math{\Pi\transp} is chosen independently to be 
  \math{\ee_i\transp/p_i} with probability \math{p_i}. Let
  \math{\Atilde=\Pi\transp\matA}. Then, for all \math{\xx\in\R^d},
  \mld{
    (1+\varepsilon)^{-1}\norm{\matA\xx}\le \norm{\Atilde\xx}\le
      (1+\varepsilon)\norm{\matA\xx}.
  }
\end{theorem}
Lemma~\ref{lemma:lewis-intro} follows by using
Theorem~\ref{theorem:lewis-general-beta} with \math{\beta} upper
bounded
by \math{12}. (\math{c} is the absolute constant \math{C_s} in
Theorem~2.3 of \cite{CP2015}).
\begin{proof}
  Let \math{\lambda_i=2c\beta\varepsilon^{-2}w_i\ln(2c\beta^2 \varepsilon^{-2}d)}.
  The \math{\lambda_i} are just rescaled versions of the \math{w_i}. Let
  \math{r=\sum_{i\in[n]}\lambda_i} and let \math{p_i=\lambda_i/r}. The \math{p_i}
  are exactly the sampling probabilities in the statement of the theorem.
  We prove that
  \mld{
    \lambda_i\ge \frac{c w_i^*}{\varepsilon^2}\ln r,\label{eq:theorem:lewis-sampling-1}}
  which means that the \math{\lambda_i} satisfy the conditions to apply
  Theorem~2.3
  in \cite{CP2015}, which gives that \math{r} samples suffice. Here,
  \eqar{
    r&=&\sum_{i\in[n]}\lambda_i\nonumber\\
    &=&2c\beta\varepsilon^{-2}\ln(2c\beta^2 \varepsilon^{-2}d)\sum_{i\in[n]}w_i
    \nonumber\\
    &\le&
    2c\beta\varepsilon^{-2}\ln(2c\beta^2 \varepsilon^{-2}d)
    \sum_{i\in[n]}\beta\cdot w_i^*\nonumber\\
    &=&
    2c\beta^2\varepsilon^{-2}d\ln(2c\beta^2 \varepsilon^{-2}d),
  }
  which proves the theorem. All that remains is to show \r{eq:theorem:lewis-sampling-1}. Let us consider \math{\lambda_i-cw_i^*\varepsilon^{-2}\ln r}:
  \eqar{
    \lambda_i-\frac{cw_i^*}{\varepsilon^2}\ln r
    &=&
    \frac{2c\beta}{\varepsilon^2}w_i\ln(2c\beta^2 \varepsilon^{-2}d)-
    \frac{cw_i^*}{\varepsilon^2}\ln r\nonumber\\
    &\ge&
    \frac{2c}{\varepsilon^2}w_i^*\ln(2c\beta^2 \varepsilon^{-2}d)-
    \frac{cw_i^*}{\varepsilon^2}\ln r\nonumber\\
    &=&
    \frac{c w_i^*}{\varepsilon^2}(2\ln(2c\beta^2 \varepsilon^{-2}d)-\ln r),
  }
  where in the middle inequality, we used
  \math{w_i\ge w_i^*/\beta}. To finish the proof, we show that the
  term in parentheses above is non-negative. Since
  \math{\ln r\le \ln(4c\beta^2\varepsilon^{-2}d\ln(c\beta \varepsilon^{-2}d))},
    we have that
    \mld{
      2\ln(2c\beta^2 \varepsilon^{-2}d)-\ln r
      \ge
      2\ln(2c\beta^2 \varepsilon^{-2}d)
      -
      \ln(2c\beta^2\varepsilon^{-2}d)
      -
      \ln\ln(2c\beta^2 \varepsilon^{-2}d)
      \ge 0.}
\end{proof}
When \math{n} is large, \math{\beta\approx 6} and
the coreset size is
\math{72c\cdot \varepsilon^{-2} d\ln(72c \varepsilon^{-2}d)}
as claimed in Lemma~\ref{lemma:lewis-intro}.
\qedsymb

%% file: FastEmbedding.bbl
\begin{thebibliography}{10}

\bibitem{A2001}
Dimitris Achlioptas.
\newblock Database-friendly random projections.
\newblock In {\em Proc. PODS}, pages 274--281, 2001.

\bibitem{AW2002}
R.~Ahlswede and A.~Winter.
\newblock Strong converse for identification via quantum channels.
\newblock {\em IEEE Trans. Inf. Theor.}, 48(3):569--579, 2002.

\bibitem{AC2006}
Nir Ailon and Bernard Chazelle.
\newblock Approximate nearest neighbors and the fast {J}ohnson-{L}indenstrauss
  transform.
\newblock In {\em Proc. STOC}, pages 557--563, 2006.

\bibitem{AC2009}
Nir Ailon and Bernard Chazelle.
\newblock The fast {J}ohnson-{L}indenstrauss transform and approximate nearest
  neighbors.
\newblock {\em SIAM J. Comput.}, 39(1):302--322, 2009.

\bibitem{AL2013}
Nir Ailon and Edo Liberty.
\newblock An almost optimal unrestricted fast johnson-lindenstrauss transform.
\newblock {\em ACM Trans. Algorithms}, 9(3):1--21, 2013.

\bibitem{AR2014}
Nir Ailon and Holger Rauhut.
\newblock Fast and rip-optimal transforms.
\newblock {\em Discrete {\&} Computational Geometry}, 52(4):780--798, 2014.

\bibitem{A1930}
Herman Auerbach.
\newblock {\em On the area of convex curves with conjugate diameters}.
\newblock PhD thesis, University of Lw\'{o}w, 1930.

\bibitem{BSS2009}
Joshua~D. Batson, Daniel~A. Spielman, and Nikhil Srivastava.
\newblock Twice-ramanujan sparsifiers.
\newblock In {\em Proc. STOC}, pages 255--262, 2009.

\bibitem{BLM1989}
J.~Bourgain, J.~Lindenstrauss, and V.~Milman.
\newblock Approximation of zonoids by zonotopes.
\newblock {\em Acta Mathematica}, 162(1):73--141, 1989.

\bibitem{malik188}
Kenneth Clarkson, Petros Drineas, Malik Magdon-Ismail, Michael Mahoney,
  Xiangrui Meng, and David Woodruff.
\newblock The fast {C}auchy transform and faster robust linear regression.
\newblock In {\em Proceedings of ACM-SIAM Symposium on Discrete Algorithms
  (SODA)}, January 2013.

\bibitem{malik209}
Kenneth Clarkson, Petros Drineas, Malik Magdon-Ismail, Michael Mahoney,
  Xiangrui Meng, and David Woodruff.
\newblock The fast {C}auchy transform and faster robust linear regression.
\newblock {\em {SIAM} Journal on Computing}, 45(3):763--810, January 2016.

\bibitem{CW2013}
Kenneth~L. Clarkson and David~P. Woodruff.
\newblock Low rank approximation and regression in input sparsity time.
\newblock In {\em Proc. STOC}, 2013.

\bibitem{CP2015}
Michael~B. Cohen and Richard Peng.
\newblock \math{\ell_p} row sampling by lewis weights.
\newblock In {\em Proc STOC}, pages 183--192, 2015.

\bibitem{DDHKM2009}
Anirban Dasgupta, Petros Drineas, Boulos Harb, Ravi Kumar, and Michael~W.
  Mahoney.
\newblock Sampling algorithms and coresets for \math{\ell_p} regression.
\newblock {\em {SIAM} Journal on Computing}, 38(5):2060--2078, 2009.

\bibitem{DKM2006a}
Petros Drineas, Ravi Kannan, and Michael~W. Mahoney.
\newblock Fast monte carlo algorithms for matrices i: Approximating matrix
  multiplication.
\newblock {\em SIAM J. Computing}, 36(1):132--157, 2006.

\bibitem{DKM2006b}
Petros Drineas, Ravi Kannan, and Michael~W. Mahoney.
\newblock Fast monte carlo algorithms for matrices ii: Computing a low-rank
  approximation to a matrix.
\newblock {\em SIAM J. Computing}, 36(1):158--183, 2006.

\bibitem{DKM2006c}
Petros Drineas, Ravi Kannan, and Michael~W. Mahoney.
\newblock Fast monte carlo algorithms for matrices ii: Computing a low-rank
  approximation to a matrix.
\newblock {\em SIAM J. Computing}, 36(1):184--206, 2006.

\bibitem{malik186}
Petros Drineas, Malik Magdon-Ismail, Michael Mahoney, and David Woodruff.
\newblock Fast approximation of matrix coherence and statistical leverage.
\newblock {\em Journal of Machine Learning Research (JMLR)}, 13:3441--3472,
  2012.

\bibitem{DMM2006}
Petros Drineas, Michael~W. Mahoney, and S.~Muthukrishnan.
\newblock Sampling algorithms for l2 regression and applications.
\newblock In {\em Proc. SODA}, pages 1127--1136, 2006.

\bibitem{DD1996}
Devdatt Dubhashi and Desh Ranjan.
\newblock Balls and bins: A study in negative dependence.
\newblock BRICS Report Series, 1996.

\bibitem{GV96}
G.~Golub and C.~van Loan.
\newblock {\em Matrix computations}.
\newblock The Johns Hopkins University Press, London, 3 edition, 1996.

\bibitem{JL1984}
W.~B. Johnson and J.~Lindenstrauss.
\newblock Extensions of lipschitz mappings into a hilbert space.
\newblock {\em Contemp. Math.}, 26:189--206, 1984.

\bibitem{K1995}
Philip~A. Knight.
\newblock Fast rectangular matrix multiplication and qr decomposition.
\newblock {\em Linear Algebra and Its Applications}, pages 69--81, 1995.

\bibitem{KW2011}
Felix Krahmer and Rachel Ward.
\newblock New and improved johnson–lindenstrauss embeddings via the
  restricted isometry property.
\newblock {\em SIAM J. Math. Anal.}, 43(3):1269--1281, 2011.

\bibitem{LMP2013}
M.~Li, G.~L. Miller, and R.~Peng.
\newblock Iterative row sampling.
\newblock In {\em Proc. FOCS}, pages 127--136, 2013.

\bibitem{LLR1995}
N.~Linial, E.~London, and Y.~Rabinovich.
\newblock he geometry of graphs and some of its algorithmic applications.
\newblock {\em Combinatorica}, 15:215--245, 1995.

\bibitem{malik145}
Malik Magdon-Ismail.
\newblock Row sampling for matrix algorithms via a non-commutative bernstein
  bound.
\newblock {\em arXiv preprint: arXiv:1008.0587v1}, 2010.

\bibitem{M2003}
Andreas Maurer.
\newblock A bound on the deviation probability for sums of non-negative random
  variables.
\newblock {\em J. Inequalities in Pure and Applied Mathematics}, 4(1):15, 2013.

\bibitem{MM2013}
Xiangrui Meng and Michael~W. Mahoney.
\newblock Low-distortion subspace embeddings in input-sparsity time and
  applications to robust linear regression.
\newblock In {\em Proc. STOC}, pages 91--100, 2013.

\bibitem{NN2013}
Jelani Nelson and Huy~L. Nguy\math{\tilde{\hat{\text{e}}}}n.
\newblock Osnap: Faster numerical linear algebra algorithms via sparser
  subspace embeddings.
\newblock In {\em Proc. FOCS}, pages 117--126, 2013.

\bibitem{RV2007}
Mark Rudelson and Roman Vershynin.
\newblock Sampling from large matrices: An approach through geometric
  functional analysis.
\newblock {\em Journal of the ACM}, 2007.

\bibitem{S2006}
T.~Sarlos.
\newblock Improved approximation algorithms for large matrices via random
  projections.
\newblock In {\em Proc. FOCS}, pages 143--152, 2006.

\bibitem{SW2011}
Christian Sohler and David~P. Woodruff.
\newblock Subspace embeddings for the l1-norm with applications.
\newblock In {\em Proc. STOC}, pages 755--764, 2011.

\bibitem{T1990}
Michel Talagrand.
\newblock Embedding subspaces of \math{\ell_1} into \math{\ell_1^n}.
\newblock {\em Proc. American Mathematical Society}, 108(2):363--369, 1990.

\bibitem{T2011}
J.~A. Tropp.
\newblock Improved analysis of the subsampled randomized hadamard transform.
\newblock {\em Adv. Adapt. Data Anal.}, 3(1--2):115--126, 2011.

\bibitem{T2012}
Joel~A. Tropp.
\newblock User-friendly tail bounds for sums of random matrices.
\newblock {\em Foundations of Computational Mathematics}, 4:389--434, 2012.

\bibitem{WW2018}
Ruosong Wang and David~P. Woodruff.
\newblock Tight bounds for \math{\ell_p} oblivious subspace embeddings.
\newblock {\em https://arxiv.org/abs/1801.04414}, 2018.
\newblock (arXiv preprint).

\bibitem{W2012}
Virginia~Vassilevska Williams.
\newblock Multiplying matrices faster than {C}oppersmith-{W}inograd.
\newblock In {\em Proc. STOC}, pages 887--898, 2012.

\bibitem{WZ2013}
David Woodruff and Qin Zhang.
\newblock {Subspace Embeddings and $\ell_p$-Regression Using Exponential Random
  Variables}.
\newblock In {\em Proc. of Conf. on Learning Theory}, pages 546--567, 2013.

\bibitem{W2014}
David~P. Woodruff.
\newblock {\em Sketching as a Tool for Numerical Linear Algebra}, volume 10:
  1--2 of {\em Foundations and Trends in Theoretical Computer Science}.
\newblock NOW, 2014.

\bibitem{YMM2014}
Jiyan Yang, Xiangrui Meng, and Michael~W. Mahoney.
\newblock Quantile regression for large-scale applications.
\newblock {\em SIAM Journal on Scientific Computing}, 36(5):S78--S110, 2014.

\end{thebibliography}
